%% file: samplepaper.tex
\documentclass[runningheads]{llncs}
\usepackage{graphicx}
\usepackage{amsmath,amssymb,amsfonts}
\usepackage{booktabs}   %% For formal tables:
\usepackage{subcaption} %% For complex figures with subfigures/subcaptions
\usepackage{url}            % simple URL typesetting
\usepackage{amsfonts}       % blackboard math symbols
\usepackage{nicefrac}       % compact symbols for 1/2, etc.
\usepackage{color}
\usepackage{xcolor}
\usepackage{graphicx}
\usepackage[percent]{overpic}
\usepackage{bbm}
\usepackage{gensymb}

\usepackage{microtype}
\usepackage{mathtools}
\usepackage{wrapfig}
\usepackage{multirow}

\usepackage{hyperref}
\usepackage{algorithm}
\usepackage{algpseudocode}
\algrenewcommand\algorithmicrequire{\textbf{Input:}}
\algrenewcommand\algorithmicensure{\textbf{Output:}}
\newcommand{\hide}[1]{}
\DeclarePairedDelimiter{\nint}\lfloor\rceil

\newcommand\blfootnote[1]{%
  \begingroup
  \renewcommand\thefootnote{}\footnote{#1}%
  \addtocounter{footnote}{-1}%
  \endgroup
}

\newdimen\figrasterwd
\figrasterwd\textwidth

\begin{document}

\title{Correctness Verification of Neural Networks}
%
%\titlerunning{Abbreviated paper title}
% If the paper title is too long for the running head, you can set
% an abbreviated paper title here
%

\author{Yichen Yang \and
Martin Rinard 
}
% First names are abbreviated in the running head.
% If there are more than two authors, 'et al.' is used.
%
\institute{Department of Electrical Engineering and Computer Science, \\
Massachusetts Institute of Technology, USA \\
\email{\{yicheny,rinard\}@csail.mit.edu}}

\maketitle              % typeset the header of the contribution

\begin{abstract}

We present a novel framework for specifying and verifying correctness globally for neural networks on perception tasks. Most previous works on neural network verification for perception tasks focus on robustness verification. Unlike robustness verification, which aims to verify that the prediction of a network is {\em stable} in some {\em local} regions around labelled points, our framework provides a way to specify correctness {\em globally} in the whole target input space and verify that the network is {\em correct} for all target inputs (or find the regions where the network is not correct). We provide a {\em specification} through 1) a {\em state space} consisting of all relevant states of the world and 2) an {\em observation process} that produces neural network inputs from the states of the world. Tiling the state and input spaces with a finite number of tiles, obtaining ground truth bounds from the state tiles and network output bounds from the input tiles, then comparing the ground truth and network output bounds delivers an upper bound on the network output error for any inputs of interest. The presented framework also enables detecting illegal inputs -- inputs that are not contained in (or close to) the target input space as defined by the state space and observation process (the neural network is not designed to work on them), so that we can flag when we don't have guarantees. Results from two case studies highlight the ability of our technique to verify error bounds over the whole target input space and show how the error bounds vary over the state and input spaces. \blfootnote{A shorter version of this paper is published in NeurIPS 2019 Workshop on Machine Learning with Guarantees.}

\end{abstract}
\section{Introduction}

Neural networks are now recognized as powerful function approximators with impressive performance across a wide range of applications. A large body of work focuses on providing formal guarantees for neural networks \cite{ai2-2018,Reluplex2017,ConvDual2018,tjeng2019,singh2018robustness,control2018,fairsquare-17}. For perception tasks (e.g. vision, speech recognition), most work has focused on robustness verification \cite{ai2-2018,Reluplex2017,ConvDual2018,tjeng2019,singh2018robustness}, which aims to verify if the network prediction is {\em stable} for all inputs in some neighborhood around a selected input point. However, there is yet no systematic way to specify, then verify, that the network is {\em correct} (within a specified tolerance) {\em globally} over target regions of the input space of interest.

We present the first framework for specifying and verifying the correctness of neural networks on perception tasks. Neural networks are often used to predict some property of the world given an observation such as an image or audio recording. We provide a {\em specification} through 1) a {\em state space} consisting of a symbolic representation of all target states of the world and 2) an {\em observation process} that produces neural network inputs from these target states. Then the target inputs of interest are all inputs that can be observed from the state space via the observation process. We define the set of target inputs as the {\em target input space}. Because the quantity that the network predicts is some property of the observed state of the world, the state defines the ground truth output (and therefore defines the correct output for each target input to the neural network).

We present \textit{Tiler}, an algorithm for correctness verification of neural networks. Evaluating the correctness of the network on a single state is straightforward --- use the observation process to obtain the possible inputs for that state, use the neural network to obtain the possible outputs, then compare the outputs to the ground truth from the state. To do correctness verification, we generalize this idea to work with {\em tiled} state and input spaces. We cover the state and input spaces with a finite number of \textit{tiles}: each state tile comprises a set of states; each input tile is the image of the corresponding state tile under the observation process. The state tiles provide ground truth bounds for the corresponding input tiles. We use recently developed techniques from the robustness verification literature to obtain network output bounds for each input tile~\cite{maxsens2018,ai2-2018,Fastlin2018,ILP2016,NSVerify2017,tjeng2019,eevbnn}. A comparison of the ground truth and network output bounds delivers an error upper bound for that region of the state space. The error bounds for all the tiles jointly provide the correctness verification result.

The proposed framework also enables detecting illegal inputs -- inputs that are not within (or close to) the target input space given by the specification. The neural network is not designed to work on these inputs, so we can flag them when they occur in runtime. The result is that the client or user of the neural network knows when the correctness guarantee holds or may not hold. By utilizing neural network prediction to guide the search, we are able to speed up the detecting process significantly.

We present two case studies. The first involves a world with a (idealized) fixed road and a camera that can vary its horizontal offset and viewing angle with respect to the centerline of the road (Section~\ref{sec:case-study-1}). The state of the world is therefore characterized by the offset $\delta$ and the viewing angle $\theta$. A neural network takes the camera image as input and predicts the offset and the viewing angle. The state space includes the $\delta$ and $\theta$ of interest. The observation process is the camera imaging process, which maps camera positions to images. This state space and the camera imaging process provide the specification. The target input space is the set of camera images that can be observed from all camera positions of interest. For each image, the camera positions of all the states that can produce the image give the possible ground truths. We tile the state space using a grid on $(\delta, \theta)$. Each state tile gives a bound on the ground truth of $\delta$ and $\theta$. We then apply the observation process to project each state tile into the image space. We compute a bounding box for each input tile and apply techniques from robustness verification~\cite{tjeng2019} to obtain neural network output bounds for each input tile. Comparing the ground truth bounds and the network output bounds gives upper bounds on network prediction error for each tile. We verify that our trained neural network provides good accuracy across the majority of the state space of interest and bound the {\em maximum} error the network will ever produce on any target input.

The second case study verifies a neural network that classifies a LiDAR measurement of a sign in an (idealized) scene into one of three shapes (Section~\ref{sec:case-lidar}). The state space includes the position of the LiDAR sensor and the shape of the sign. We tile the state space, project each tile into the input space via the LiDAR observation process, and again apply techniques from robustness verification to verify the network, including identifying regions of the state and input space where the network may deliver an incorrect classification.

\subsection{Contributions}
This paper makes the following contributions:

\paragraph{Specification:}
We show how to use state spaces and observation processes to specify the correctness of neural networks for perception - the specification identifies the target inputs of interest and the correct output for each target input. To our knowledge, this is the first systematic approach to provide a global correctness specification for perception neural networks.

\paragraph{Verification:} We present an algorithm, \textit{Tiler}, for correctness verification. With state spaces and observation processes providing the specification, this is the first algorithm (to our knowledge) for verifying that a neural network produces the {\em correct} output (up to a specified tolerance) for {\em every} target input of interest. The algorithm can also compute tighter correctness bounds for focused regions of the state and input spaces.

\paragraph{Detecting inputs outside the target input space:} We show how to use the specification to detect and flag inputs that are not within the verified target input space.

\paragraph{Case Studies:} We apply the framework to specify and verify neural networks that predicts camera position from an input image and neural networks that classifies the shape of the sign from input LiDAR measurements.

\section{Related Work}
\label{sec:related-work}

Motivated by the vulnerability of neural networks to adversarial attacks \cite{adversarial2016,adversarial2014}, researchers have developed a range of techniques for defense via adversarial training \cite{Goodfellow2014,Madry2017,Zhu2017Adv,Jin2020}, and further developed techniques for verifying robustness --- they aim to verify if the neural network prediction is stable in some neighborhoods around selected input points. \cite{salman2019convex,verify-survey-2019} provide an overview of the field. A range of approaches have been explored, including layer-by-layer reachability analysis \cite{exactreach-2017,maxsens2018} with abstract interpretation \cite{ai2-2018} or bounding the local Lipschitz constant \cite{Fastlin2018}, formulating the network as constraints and solving the resulting constrained optimization problem \cite{ILP2016,NSVerify2017,Cheng2017,tjeng2019,singh2018robustness,NIPS2019Singh}, solving the dual problem \cite{duality2018,ConvDual2018,certify2018}, and formulating and solving using SMT/SAT solvers \cite{Reluplex2017,Planet2017,Huang2017,eevbnn}. When the adversarial region is large, several techniques divide the domain into smaller subdomains and verify each of them \cite{Singh2019,bunel2017unified}.
Unlike this prior research, which aims to verify that the prediction of a network is {\em stable} in some {\em local} regions around labelled points, the research presented in this paper introduces a framework to verify that a neural network for perception computes {\em correct} outputs within a specified tolerance {\em globally} over a target input space as specified by a symbolic state space and an observation process.

Besides robustness against norm-bounded perturbations, several works also explore the verification/certification of robustness against other types of perturbations, e.g. geometric transformations \cite{CertGeometric2019,Singh2019}. Unlike this prior research, which aims to verify that the prediction of a network does not change in a {\em local} region defined by some transformations on the {\em input} to the neural network, the research presented in this paper introduces {\em state space} and {\em observation process} to specify and verify correctness over the {\em whole target input space}.

Besides robustness, researchers have also explored verification of several other properties. One line of work focuses on properties of systems using neural network controllers. \cite{control2018} presents an approach that computes the reachable states and verifies region stability for closed-loop systems with neural network controllers. \cite{control-safety2018} verifies safety by checking that all reachable states are safe. \cite{verisig-18} transforms the neural network into an equivalent hybrid system and verifies closed-loop properties. \cite{art-19} further integrates the verification/certification results into training of the neural network controllers. \cite{shield-19} synthesizes a deterministic program as a shield to the neural network controller and verifies the safety of the combined system.

Another property explored by prior work is fairness. \cite{fairsquare-17} verifies fairness properties of machine learning models using numerical integration. \cite{fairness-19} uses adaptive sampling to obtain an algorithm that has better scalability.

In both the controller and fairness settings, the properties to verify can be directly specified in terms of input-output relationships of the neural network, since the input to the neural network has explicit semantic meanings: in controller settings the input corresponds to the system state, and in fairness settings the input contains explicit features such as gender or race. The research presented in this paper, on the other hand, focuses on perception, where the correctness property cannot be specified directly using input-output relationships. We therefore introduce the state space and observation process to make such specification (and verification) feasible.

Prior work on neural network testing focuses on constructing better test cases to expose problematic network behaviors. Researchers have developed approaches to build test cases that improve coverage on possible states of the neural network, for example neuron coverage \cite{DeepXplore2017,DeepTest2018} and generalizations to multi-granular coverage \cite{DeepGauge2018} and MC/DC \cite{MCDC2001} inspired coverage \cite{MCDCtest2018}. \cite{TensorFuzz2018} presents coverage-guided fuzzing methods for testing neural networks using the above coverage criteria. \cite{DeepTest2018} generates realistic test cases by applying natural transformations (e.g. brightness change, rotation, add rain) to seed images. \cite{okelly2018scalable} uses simulation to test autonomous driving systems with deep learning based perception. \cite{scenic-pldi19} designs a probabilistic programming language to specify scenes that helps training and testing neural perception systems. Unlike this prior research, which focuses on generating useful test cases for the neural network, the research presented in this paper verifies the correctness of the neural network for all inputs in the target input space.

\section{Correctness Verification of Neural Networks}
\label{sec:formalization}

Consider the general perception problem of taking an input observation $x$ and to predict some quantity of interest $y$. It can be a regression problem (continuous $y$) or a classification problem (discrete $y$). Some neural network model is trained for this task. We denote its function by $f: \mathcal{X} \rightarrow \mathcal{Y}$, where $\mathcal{X}$ is the space of all possible inputs to the neural network and $\mathcal{Y}$ is the space of all possible outputs. Behind the input observation $x$ there is some state of the world $s$. Denote $\mathcal{S}$ as the space of all states of the world that the network is expected to work in. For each state of the world, a set of possible inputs can be observed. We denote this {\em observation process} using a mapping $g: \mathcal{S} \rightarrow \mathcal{P}(\tilde{\mathcal{X}})$,
where $g(s)$ is the set of inputs that can be observed from $s$. Here $\mathcal{P}(\cdot)$ is the power set, and $\tilde{\mathcal{X}} \subseteq \mathcal{X}$ is the \textit{target input space}, the part of input space that may be observed from the {\em state space} $\mathcal{S}$. Concretely, $\tilde{\mathcal{X}} = \{ x | \exists s \in \mathcal{S}, x \in g(s) \}$.

The quantity of interest $y$ is some attribute of the state of the world. We denote the ground truth of $y$ using a function $\lambda:\mathcal{S}\rightarrow \mathcal{Y}$. This specifies the ground truth for each input, which we denote as a mapping $\hat{f}: \tilde{\mathcal{X}} \rightarrow \mathcal{P}(\mathcal{Y})$. $\hat{f}(x)$ is the set of possible ground truth values of $y$ for a given $x$:
    \begin{equation}
    \hat{f}(x) = \{ y | \exists s \in \mathcal{S}, y=\lambda(s), x \in g(s) \}.
    \end{equation}

The target input space $\tilde{\mathcal{X}}$ and the ground truth mapping $\hat{f}$ together form a {\em specification}. In general, we cannot compute and represent $\tilde{\mathcal{X}}$ and $\hat{f}$ directly --- indeed, one purpose of the neural network is to compute an approximation to this ground truth $\hat{f}$ which, in general, is not available given only the input $x$. $\tilde{\mathcal{X}}$ and $\hat{f}$ are instead determined implicitly by $\mathcal{S}$, $g$, and $\lambda$.

The error of the neural network is then characterized by the difference between $f$ and $\hat{f}$. Concretely, the maximum possible error at a given input $x \in \tilde{\mathcal{X}}$ is:
    \begin{equation}
    \label{eqn:error}
    e(x) = \max_{y\in \hat{f}(x)} d(f(x), y),
    \end{equation}
where $d(\cdot, \cdot)$ is some measurement on the size of the error between two values of the quantity of interest. In this paper, for regression, we consider the absolute value of the difference $d(y_1, y_2) = |y_1 - y_2|$.\footnote{For clarity, we formulate the problem with a one-dimensional quantity of interest. Extending to multidimensional output (multiple quantities of interest) is straightforward: we treat the prediction of each output dimension as using a separate neural network, all of which are the same except the final output layer.} For classification, we consider a binary error measurement $d(y_1, y_2) = \mathbbm{1}_{y_1 \neq y_2}$ (indicator function), i.e. the error is 0 if the prediction is correct, 1 if the prediction is incorrect. Other error measures can also be used depending on the requirement.

The goal of correctness verification is to verify that the maximum error in the network prediction with respect to the specification is less than some threshold $\Delta$ (determined by the requirement from the system in which the neural network is used). We formulate the problem of correctness verification formally here:

\textit{Given a trained neural network $f$ and a specification $(\tilde{\mathcal{X}}, \hat{f})$ determined implicitly by $\mathcal{S}$, $g$, and $\lambda$, verify that error $e(x) \leq \Delta$ for any target input $x\in \tilde{\mathcal{X}}$.}

In addition, our research also applies to the problem of verifying correctness in more localized or focused regions of the target input space. It is not always possible to have a neural network that is correct in the whole target input space, so being able to verify subregions of the target input space and identify if an input is within such subregions is an important problem.

\section{\textit{Tiler}}
\label{sec:framework}

We next present \textit{Tiler}, an algorithm for correctness verification of neural networks.
We present here the algorithm for regression settings, with sufficient conditions for the verification to be sound. The algorithm for classification settings is similar (see Appendix \ref{sec:classification}).

\noindent \textbf{Step 1:} Divide the state space $\mathcal{S}$ into \textit{state tiles} $\{\mathcal{S}_i \}$ such that $\cup_i \mathcal{S}_i  = \mathcal{S}$.

The image of each $\mathcal{S}_i$ under $g$ gives an \textit{input tile} (a tile in the input space): $\mathcal{X}_i = \{ x | x \in g(s), s \in \mathcal{S}_i \}$. The resulting tiles $\{\mathcal{X}_i\}$ satisfy the following condition:

\textit{Condition 4.1.} $\tilde{\mathcal{X}} \subseteq \cup_i \mathcal{X}_i$.
\\

\noindent \textbf{Step 2:} For each $\mathcal{S}_i$, compute the ground truth bound as an interval $[l_i, u_i]$, such that $\forall s \in \mathcal{S}_i, l_i \leq \lambda(s) \leq u_i$.

The bounds computed this way satisfy the following condition, which (intuitively) states that the possible ground truth values for an input point must be covered jointly by the ground truth bounds of all the input tiles that contain this point:

\textit{Condition 4.2(a).} For any $x \in \tilde{\mathcal{X}}, \forall y \in \hat{f}(x), \exists \mathcal{X}_i $ such that $ x\in \mathcal{X}_i$ and $l_i \leq y \leq u_i$.
\\

\noindent Previous research has produced a variety of methods that bound the neural network output over given input regions. Examples include layer-by-layer reachability analysis \cite{maxsens2018,ai2-2018,Fastlin2018} and formulating constrained optimization problems \cite{ILP2016,NSVerify2017,tjeng2019}. Another type of approach verifies that the neural network output is within some range over given input regions by formulating and solving SAT/SMT formulae \cite{Reluplex2017,Huang2017,eevbnn}. Each method typically works for certain classes of networks (e.g. piece-wise linear networks) and certain classes of input regions (e.g. polytopes). For each input tile $\mathcal{X}_i$, we therefore introduce a bounding box $\mathcal{B}_i$ that 1) includes $\mathcal{X}_i$ and 2) is supported by the solving method:

\noindent \textbf{Step 3:} Using $\mathcal{S}_i$ and $g$, compute a bounding box $\mathcal{B}_i$ for each tile $\mathcal{X}_i = \{ x | x \in g(s), s \in \mathcal{S}_i \}$.

The bounding boxes $\mathcal{B}_i$'s must satisfy the following condition:

\textit{Condition 4.3.} $\forall i, \mathcal{X}_i \subseteq \mathcal{B}_i$.
\\

\noindent \textbf{Step 4:} Given $f$ and bounding boxes $\{ \mathcal{B}_i \}$, use an appropriate solver to solve for the network output ranges $\{ [l'_i, u'_i] \}$.

The neural network has a single output entry for each quantity of interest. Denote the value of the output entry as $o(x)$, $f(x) = o(x)$. The network output bounds $(l'_i, u'_i)$ returned by the solver must satisfy the following condition:

\textit{Condition 4.4(a)} $\forall x \in \mathcal{B}_i, l'_i \leq o(x) \leq u'_i$.
\\

\noindent \textbf{Step 5:} For each tile, use the ground truth bound $(l_i, u_i)$ and network output bound
$(l'_i, u'_i)$ to compute the error bound $e_i$:
\begin{equation}
\label{eqn:regression-tile-error}
    e_i = \max (u'_i-l_i, u_i-l'_i).
\end{equation}

$e_i$ gives the upper bound on prediction error when the state of the world $s$ is in $\mathcal{S}_i$. This is because $(l_i, u_i)$ covers the ground truth values in $\mathcal{S}_i$, and $(l'_i, u'_i)$ covers the possible network outputs for all inputs that can be generated from $\mathcal{S}_i$.
From these error bounds $\{e_i\}$, we compute a global error bound:
\begin{equation}
\label{eqn:regression-global}
    e_{\text{global}} = \max_{i} e_i.
\end{equation}

We can also compute a local error bound for any target input $x\in \tilde{\mathcal{X}}$:
\begin{equation}
\label{eqn:regression-local}
    e_{\text{local}}(x) = \max_{\{i | x\in \mathcal{B}_i\}} e_i.
\end{equation}

We store the bounding boxes $\mathcal{B}_i$'s as part of the results computed by \textit{Tiler}, so it is convenient to check containment of x in $\mathcal{B}_i$'s.

These error bounds can then be compared with the threshold $\Delta$ to decide whether the network is verified or not. Notice that we can also use SAT/SMT based approaches. For these approaches, we will combine step 4 and 5 and formulate the verification target as clauses on network output based on ground truth bounds $[l_i, u_i]$ and error threshold $\Delta$.

\begin{theorem}[Soundness of local error bound for regression]
\label{theorem-regression-local}
Given that Condition 4.1, 4.2(a), 4.3, and 4.4(a) are satisfied, then $\forall x \in \tilde{\mathcal{X}}$, $e(x) \leq e_{\text{local}}(x)$, where $e(x)$ is defined in Equation \ref{eqn:error} and $e_{\text{local}}(x)$ is computed from Equation \ref{eqn:regression-tile-error} and \ref{eqn:regression-local}.
\end{theorem}

\begin{theorem}[Soundness of global error bound for regression]
\label{theorem-regression-global}
Given that Condition 4.1, 4.2(a), 4.3, and 4.4(a) are satisfied, then $\forall x \in \tilde{\mathcal{X}}$, $e(x) \leq e_{\text{global}}$, where $e(x)$ is defined in Equation \ref{eqn:error} and $e_{\text{global}}$ is computed from Equation \ref{eqn:regression-tile-error} and \ref{eqn:regression-global}.
\end{theorem}

The proofs are given in Appendix \ref{sec:proof-append}.

%\subsection{Algorithm}
%\label{sec:formal-algorithm}

Algorithm \ref{alg:tiler} formally presents the \textit{Tiler} algorithm for regression. The implementations of \textproc{DivideStateSpace}, \textproc{GetGroundTruthBound}, and \textproc{GetBoundingBox} are problem dependent. The choice of \textproc{Solver} needs to be compatible with $\mathcal{B}_i$ and $f$. Conditions 4.1 to 4.4 specify the sufficient conditions for the returned results from these four methods such that the guarantees obtained are sound.

\begin{algorithm}
\small
\caption{Tiler (for regression)}
\label{alg:tiler}
\begin{algorithmic}[1]
\Require $\mathcal{S}, g, \lambda, f$
\Ensure $e_{\text{global}}, \{e_i \}, \{\mathcal{B}_i \}$
     \Procedure{Tiler}{$\mathcal{S}, g, \lambda, f$}
          \State $\{\mathcal{S}_i \} \gets$ \Call{DivideStateSpace}{$\mathcal{S}$} \Comment{Step 1}
          \For{each $\mathcal{S}_i$}
            \State $(l_i, u_i) \gets$ \Call{GetGroundTruthBound}{$\mathcal{S}_i, \lambda$} \Comment{Step 2}
            \State $\mathcal{B}_i \gets$ \Call{GetBoundingBox}{$\mathcal{S}_i, g$} \Comment{Step 3}
            \State $(l'_i, u'_i) \gets$ \Call{Solver}{$f, \mathcal{B}_i$} \Comment{Step 4}
            \State $e_i \gets$ $\max (u'_i-l_i, u_i-l'_i)$ \Comment{Step 5}
          \EndFor
          \State $e_{\text{global}} \gets$ {$\max (\{e_i \})$} \Comment{Step 5}
          \State \textbf{return} $e_{\text{global}}, \{e_i \}, \{\mathcal{B}_i \}$ \Comment{$\{e_i \}, \{\mathcal{B}_i \}$ can be used later to compute $e_{\text{local}}(x)$}
     \EndProcedure
\end{algorithmic}
\end{algorithm}

The complexity of this algorithm is determined by the number of tiles, which scales with the dimension of the state space $\mathcal{S}$. Because the computations for each tile are independent, our \textit{Tiler} implementation executes these computations in parallel.

\subsection{Dealing with noisy observation}
Our formulation also applies to the case of noisy observations. Notice that the observation process $g$ maps from a state to a set of possible inputs, so noise can be incorporated here. The above version of \textit{Tiler} produces \textit{hard} guarantees, i.e. the error bounds computed are valid for all cases. This works for observations with bounded noise. For cases where noise is unbounded (e.g. Gaussian noise), \textit{Tiler} can be adjusted to provide probabilistic guarantees: we compute bounding boxes $\mathcal{B}_i$ such that $P(x\in \mathcal{B}_i|x\sim g(s) , s\in \mathcal{S}_i)>1-\epsilon$ for some small $\epsilon$. Here we also need the probability measure associated with the observation process --- $g(s)$ now gives the probability distribution of input $x$ given state $s$. This will give an error bound that holds with probability at least $1-\epsilon$ for any state in this tile. We demonstrate how this is achieved in practice in the second case study (Section \ref{sec:case-lidar}).

\subsection{Detecting illegal inputs}

\textit{Tiler} provides a way to verify the correctness of the neural network over the whole target input space $\tilde{\mathcal{X}}$. We also provide a method to detect whether a new observed input is within $\tilde{\mathcal{X}}$ (the network is designed to work for it, and we have guaranteed correctness) or not (the network is not designed for it, so we don't have guarantees). In general, checking containment directly with $\tilde{\mathcal{X}}$ is hard, since there is no explicit representation of it. Instead, we use the bounding boxes $\{\mathcal{B}_i \}$ returned by \textit{Tiler} as a proxy for $\tilde{\mathcal{X}}$: for a new input $x^*$, we check if $x^*$ is contained in any of the $\mathcal{B}_i$'s. Since the network output ranges computed in Step 4 cover the inputs in each $\mathcal{B}_i$, and the error bounds incorporate the network output ranges, we know that the network output will not have unexpected drastic changes in $\mathcal{B}_i$'s. This makes $\mathcal{B}_i$'s a good proxy for the space of legal inputs.

Searching through all the $\mathcal{B}_i$'s can introduce a large overhead. We propose a way to speed up the search by utilizing the network prediction and the verified error bounds. For example for regression settings, given the network prediction $y^*=f(x^*)$ and the global error bound $e_{\textrm{global}}$, we can prune the search space by discarding tiles that do not overlap with $[y^* - e_{\textrm{global}}, y^* + e_{\textrm{global}}]$ in the ground truth attribute. The idea is that we only need to search the local region in the state space that has ground truth attribute close to the prediction, since we have verified the bound for the maximum prediction error. We evaluate this detection method and the prediction-guided search in the first case study (Section \ref{sec:case-study-1}).

\subsection{Adaptive Tiling}
\label{sec:adaptive}

For the tiling step, one strategy is to have a fixed tiling where the user specifies the size for each tile. As we show in the experimental results in Section \ref{sec:results}, with fixed tiling, tile sizes affect the tradeoff between tightness of verification and verification time: with smaller tile sizes, \textit{Tiler} verifies tighter error bounds (thus higher percentage of regions) but at a cost of longer verification time. Motivated by this tradeoff, we implement an adaptive tiling strategy. We start with a coarse tile size and run \textit{Tiler}. For tiles that \textit{Tiler} fails to verify or the solver exceeds the time limit, we divide those tiles further and run \textit{Tiler} on them again. This process runs iteratively until either the tile is verified or some pre-set minimum tile size is reached. Our results show that this adaptive tiling strategy achieves a better tradeoff between verification tightness and time and alleviates the need to choose a proper tile size.

\section{Case Studies}
\label{sec:case-study}

We perform experiments on two case studies. The first predicts the position of a camera with respect to a road given images of the road. The second classifies the shape of a sign given LiDAR measurements. We present here the problem setups and our experimental methods.\footnote{Code: \href{https://github.com/yycdavid/TilerVerify}{https://github.com/yycdavid/TilerVerify}}

\subsection{Case Study 1: Position Measurement from Road Scene}
\label{sec:case-study-1}

This case study considers a world containing a road with a centerline, two side lines, and a camera taking images of the road. The camera is at a fixed height above the road, but can vary its horizontal offset and viewing angle with respect to the centerline of the road. The task of the neural network is to predict the camera position given images of the road.

\subsubsection{Scene}
\label{sec:scene-1}

Figure \ref{fig:camera-setup} presents a schematic of the scene. The state of the world $s$ is characterized by the offset $\delta$ and angle $\theta$ of the camera position. We therefore label the states as $s_{\delta,\theta}$. We consider the camera position between the range $\delta \in [-40, 40]$ (length unit of the scene, the road width from the centerline to the side lines is 50 units) and $\theta \in [-60\degree, 60\degree]$, so the target state space $\mathcal{S}=\{s_{\delta,\theta} | \delta \in [-40, 40], \theta \in [-60\degree, 60\degree] \}$.

\begin{figure*}
  \centering
  \parbox{\figrasterwd}{
    \parbox{.2\figrasterwd}{%
    \subcaptionbox{\label{fig:camera-setup}}{\includegraphics[width=\hsize]{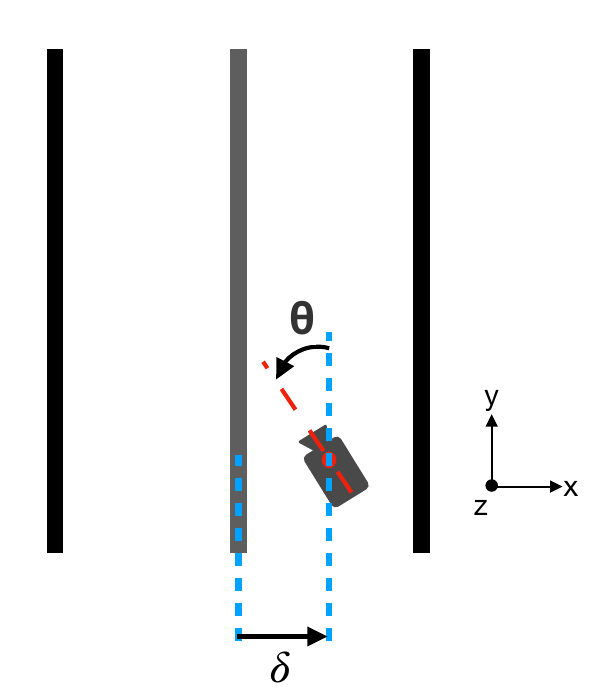}}
      \begin{center}
      \subcaptionbox{\label{fig:exp-images}}{

          \includegraphics[width=.5\hsize]{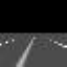}
      }
      \end{center}
    }
    \hskip1em
    \parbox{.65\figrasterwd}{%
      \subcaptionbox{\label{fig:schematics}}{\includegraphics[width=\hsize]{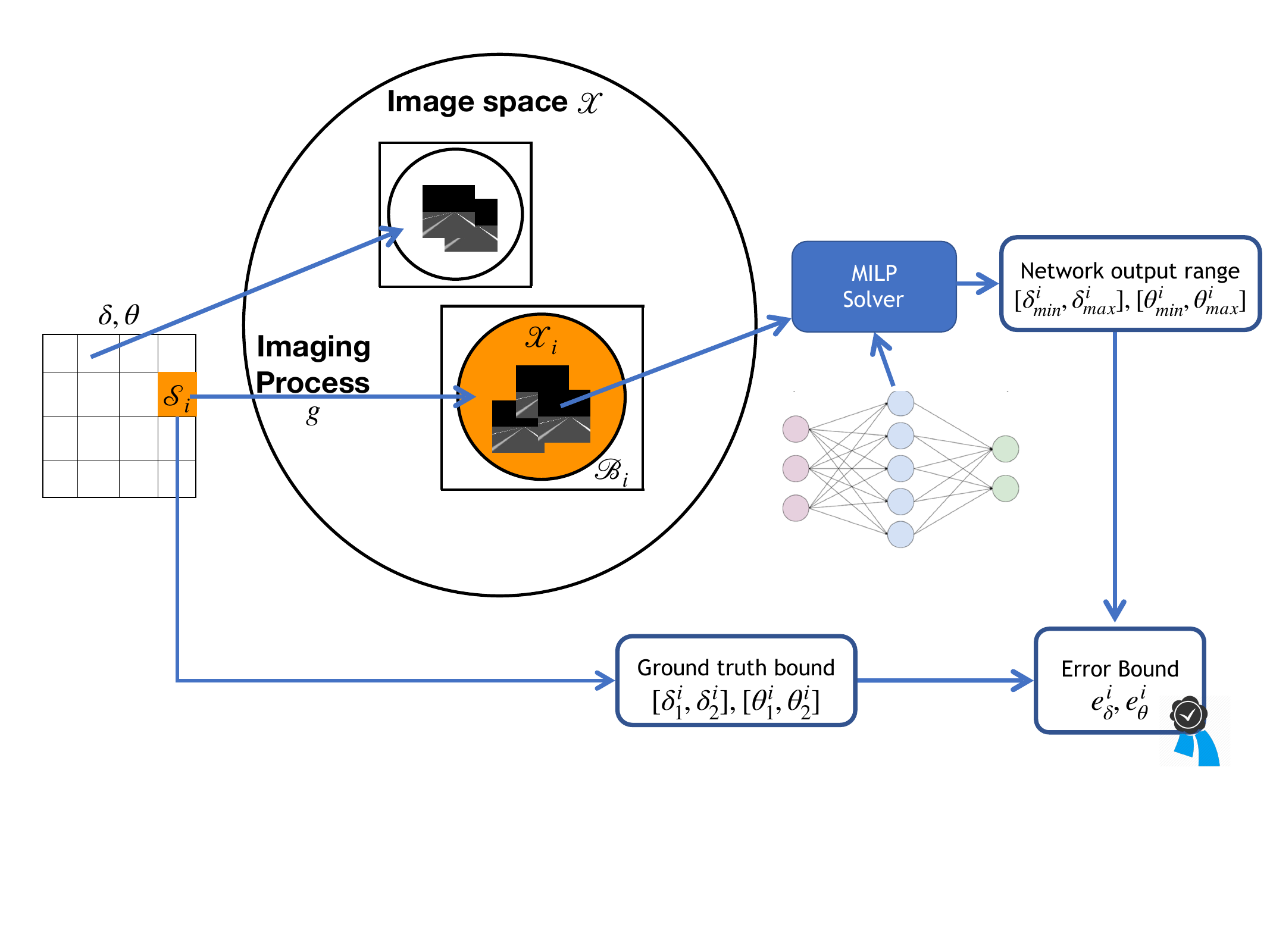}}%'height=8cm' is needed for this example only and can be dropped when using it with actual images
    }
  }
  \caption{(a) Schematic of the scene. (b) Example image taken by the camera. (c) Schematics of \textit{Tiler} on the road scene case.}
\end{figure*}

\subsubsection{Observation Process}
\label{sec:graphics-process}

The input $x$ to the neural network is the image taken by the camera. The observation process $g$ is the camera imaging process. For each pixel, we shoot a ray from the center of that pixel through the camera focal point and compute the intersection of the ray with objects in the scene. The intensity of that intersection point is taken as the intensity of the pixel. The resulting $x$'s are 32$\times$32 gray scale images with intensities in $[0, 255]$ (see Appendix \ref{sec:scene-append} and \ref{sec:camera-imaging-append} for detailed descriptions of the scene and the camera imaging process). Figure \ref{fig:exp-images} presents an example image. The target input space $\tilde{\mathcal{X}}$ is the set of all images that can be taken with $\delta \in [-40, 40]$ and $\theta \in [-60\degree, 60\degree]$.

The quantity of interest $y$ is the camera position $(\delta, \theta)$. The ground truth function $\lambda$ is simply $\lambda(s_{\delta,\theta}) = (\delta,\theta)$.

\subsubsection{Neural Network}

For the neural network, we use the same ConvNet architecture as $\small{\textbf{CNN}_{\textbf{A}}}$ in \cite{tjeng2019} and the \textit{small} network in \cite{wong2018}. It has 2 convolutional layers (size 4$\times$4, stride 2) with 16 and 32 filters respectively, followed by a fully connected layer with 100 units. All the activation functions are ReLUs. The output layer is a linear layer with 2 output nodes, corresponding to the predictions of $\delta$ and $\theta$. The network is trained on 130k images and validated on 1000 images generated from our imaging process. The camera positions of the training and validation images are sampled uniformly from the range $\delta \in [-50, 50]$ and $\theta \in [-70\degree, 70\degree]$. The network is trained with an $l_1$-loss function, using \textit{Adam} \cite{kingma2014} (see Appendix \ref{sec:training-detail} for more training details).

For error analysis, we treat the predictions of $\delta$ and $\theta$ separately. The goal is to find upper bounds on the prediction errors $e_{\delta}(x)$ and $e_{\theta}(x)$ for any target input $x\in \tilde{\mathcal{X}}$.

\subsubsection{\textit{Tiler}}
Figure \ref{fig:schematics} presents a schematic of how we apply \textit{Tiler} to this problem. Tiles are constructed by dividing $\mathcal{S}$ on $(\delta,\theta)$ into a grid of equal-sized rectangles with length $a$ and width $b$. Each cell in the grid is then $\mathcal{S}_i = \{s_{\delta,\theta} | \delta \in [\delta_1^i, \delta_2^i], \theta \in [\theta_1^i, \theta_2^i]\}$, with $\delta_2^i - \delta_1^i = a$ and $\theta_2^i - \theta_1^i = b$. Each tile $\mathcal{X}_i$ is the set of images that can be observed from $\mathcal{S}_i$. The ground truth bounds can be naturally obtained from $\mathcal{S}_i$: for $\delta$, $l_i = \delta_1^i$ and $u_i = \delta_2^i$; for $\theta$, $l_i = \theta_1^i$ and $u_i = \theta_2^i$.

We next encapsulate each tile $\mathcal{X}_i$ with an $l_{\infty}$-norm ball $\mathcal{B}_i$ by computing, for each pixel, the range of possible values it can take within the tile. As the camera position varies in a cell $\mathcal{S}_i$, the intersection point between the ray from the pixel and the scene sweeps over a region in the scene. The range of intensity values in that region determines the range of values for that pixel. We compute this region for each pixel, then find the pixel value range (see Appendix \ref{sec:compute-range} for more details on this method). The resulting $\mathcal{B}_i$ is an $l_{\infty}$-norm ball in the image space covering $\mathcal{X}_i$, represented by 32$\times$32 pixel-wise ranges.

To solve the range of outputs of the ConvNet for inputs in the $l_{\infty}$-norm ball, we adopt the MILP-based approach from \cite{tjeng2019} (Appendix \ref{sec:milp}). We adopt the same formulation but change the MILP objectives. For each $l_{\infty}$-norm ball, we solve 4 optimization problems: maximizing and minimizing the output entry for $\delta$, and another two for $\theta$. Denote the objectives solved as $\delta_{min}^i$, $\delta_{max}^i$, $\theta_{min}^i$, $\theta_{max}^i$.

We then use Equation \ref{eqn:regression-tile-error} to compute the error bounds for each tile: $e_{\delta}^i = \max (\delta_{max}^i - \delta_1^i, \delta_2^i - \delta_{min}^i)$, $e_{\theta}^i = \max (\theta_{max}^i - \theta_1^i, \theta_2^i - \theta_{min}^i)$. $e_{\delta}^i$ and $e_{\theta}^i$ give upper bounds on prediction error when the state of the world $s$ is in cell $\mathcal{S}_i$. These error bounds can later be used to compute global and local error bounds as in Equation \ref{eqn:regression-global} and \ref{eqn:regression-local}.

We experiment with both fixed and adaptive tiling strategies. For the fixed strategy, the state space is divided into tiles with a fixed size. For the adaptive strategy, the state space is initially divided into large tiles and solved by \textit{Tiler}. Then, for tiles where the maximum errors solved is larger than the required threshold or the solver exceeds its time limit, we further divide them in both $\delta$ and $\theta$ to get 4 smaller tiles and solve again. This is performed iteratively until either the maximum errors solved meet the requirement, or a minimum tile size is reached.

We implement the detector for illegal inputs by checking if the new input $x^*$ is contained in any of the bounding boxes $\mathcal{B}_i$. We compare the performance of the naive search with the state-guided search.

\subsection{Case Study 2: Shape Classification from LiDAR Sensing}
\label{sec:case-lidar}

The world in this case contains a planar sign standing on the ground. There are 3 types of signs with different shapes: square, triangle, and circle (Figure \ref{fig:sign}). A LiDAR sensor takes measurement of the scene, which is used as input to a neural network to classify the shape of the sign.

\subsubsection{Scene}
\label{sec:lidar-scene}

The sensor can vary its distance $d$ and angle $\theta$ with respect to the sign, but its height is fixed, and it is always facing towards the sign. Figure \ref{fig:lidar-schematic} shows the schematic of the set-up. Assume the working zone for the LiDAR sensor is with position $d\in[30, 60]$ (unit length) and $\theta \in [-45\degree, 45\degree]$. Then the state space $\mathcal{S}$ has 3 dimensions: two continuous ($d$ and $\theta$), and one discrete (sign shape $c$). Concretely, the state space is given by $\mathcal{S}=\{s_{d,\theta, c} | d \in [30, 60], \theta \in [-45\degree, 45\degree], c \in \{0,1,2\} \}$, where the state of the world is labeled as $s_{d,\theta, c}$ and $\{0,1,2\}$ denotes the shape classes.

\subsubsection{Observation Process}
\label{sec:lidar-model}

The observation process $g$ is the LiDAR measurement model. The LiDAR sensor emits an array of 32$\times$32 laser beams in fixed directions. The measurement from each beam is the distance to the first object hit in that direction. We consider a LiDAR measurement model where the maximum distance that can be measured is \texttt{MAX\_RANGE}=300, and the measurement has a Gaussian noise with zero mean and a standard deviation of 0.1\% of \texttt{MAX\_RANGE}. Appendix \ref{sec:lidar-scene} and \ref{sec:lidar-model} provides more details on the scene and LiDAR measurement model. 

The quantity of interest $y$ is the shape class of the sign $c$. Therefore the ground truth function is $\lambda(s_{d,\theta, c}) = c$.

\begin{figure*}
     \centering
      \parbox{\figrasterwd}{
      \parbox{.3\figrasterwd}{%
     \begin{subfigure}[b]{0.33\textwidth}
         \centering
         \includegraphics[width=\textwidth]{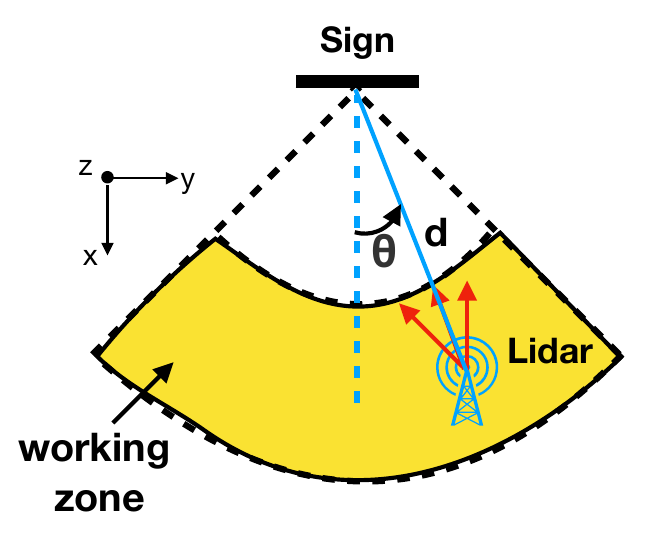}
         \vspace{-1em}
         \caption{}
         \label{fig:lidar-schematic}
     \end{subfigure}
     }
     \parbox{.3\figrasterwd}{
     \begin{subfigure}[b]{0.33\textwidth}
         \centering
         \includegraphics[width=.2\textwidth]{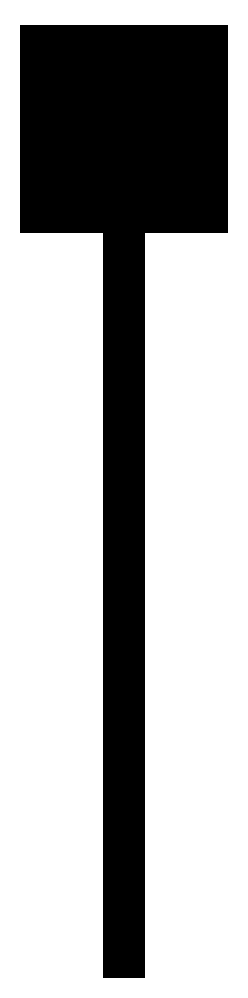}
         \includegraphics[width=.2\textwidth]{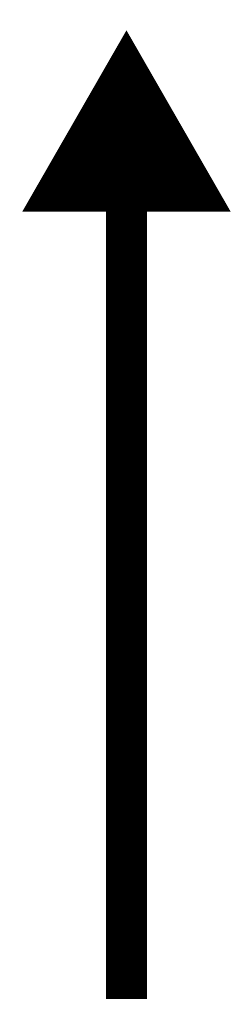}
         \includegraphics[width=.2\textwidth]{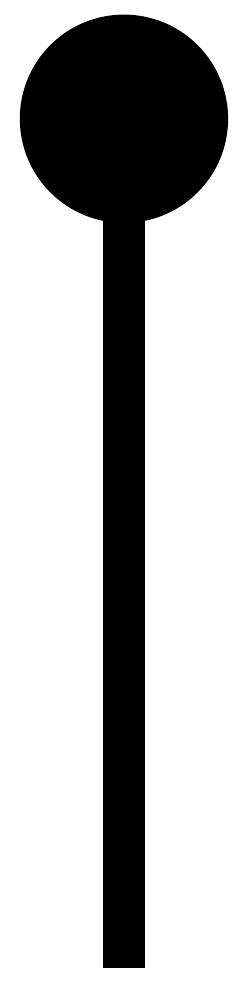}
         \caption{}
         \label{fig:sign}
     \end{subfigure}
    }\hskip-1em
    \parbox{.33\figrasterwd}{
    \begin{subfigure}[b]{0.4\textwidth}
    \centering
         \parbox{.18\figrasterwd}{
         \begin{center}
            {\includegraphics[width=.6\hsize]{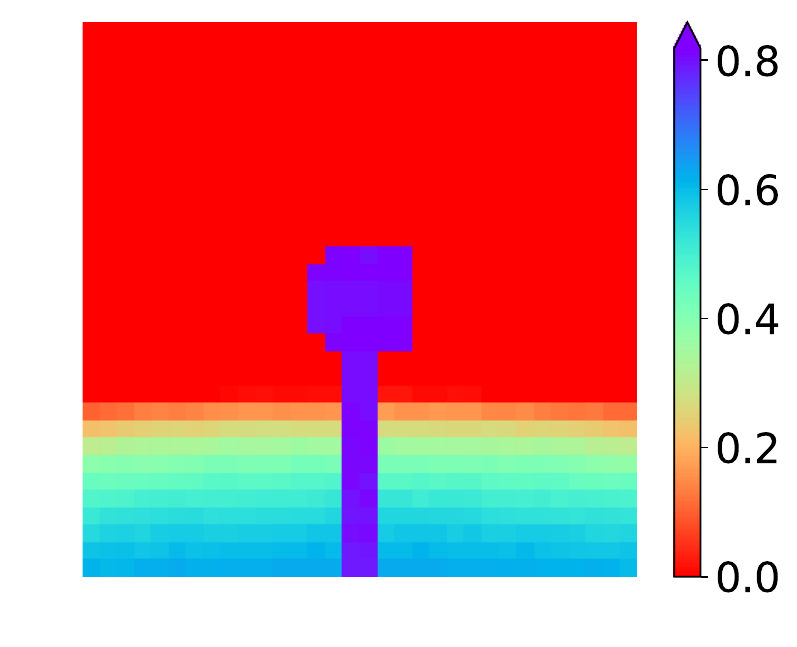}}
            {\includegraphics[width=.6\hsize]{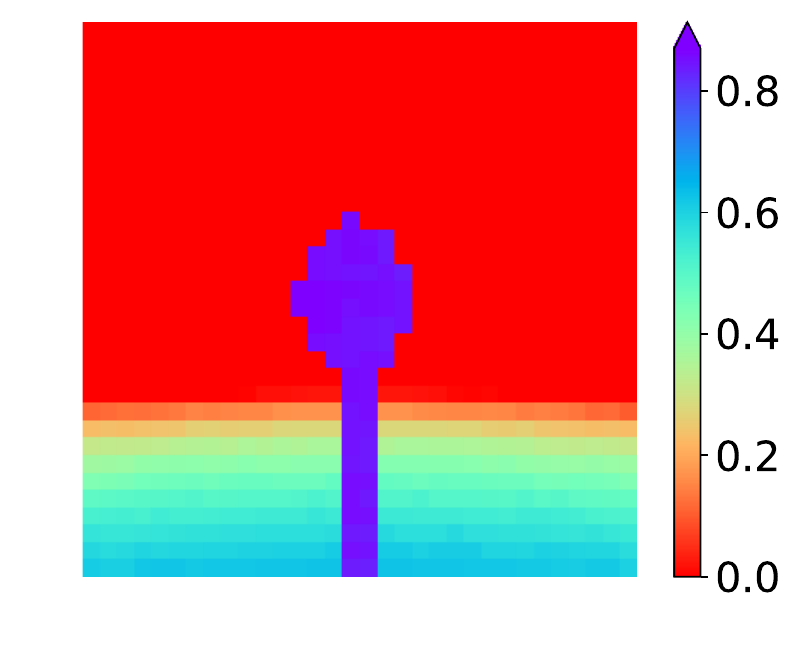}}
            {\includegraphics[width=.6\hsize]{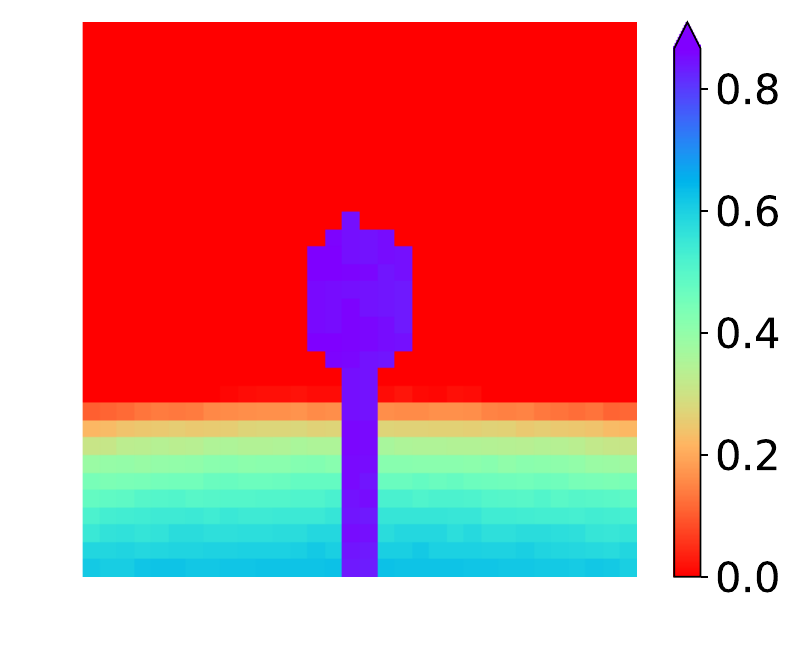}}
        \end{center}
         }\hskip-1em
         \parbox{.15\figrasterwd}{
         \begin{center}
         {\includegraphics[width=.8\hsize]{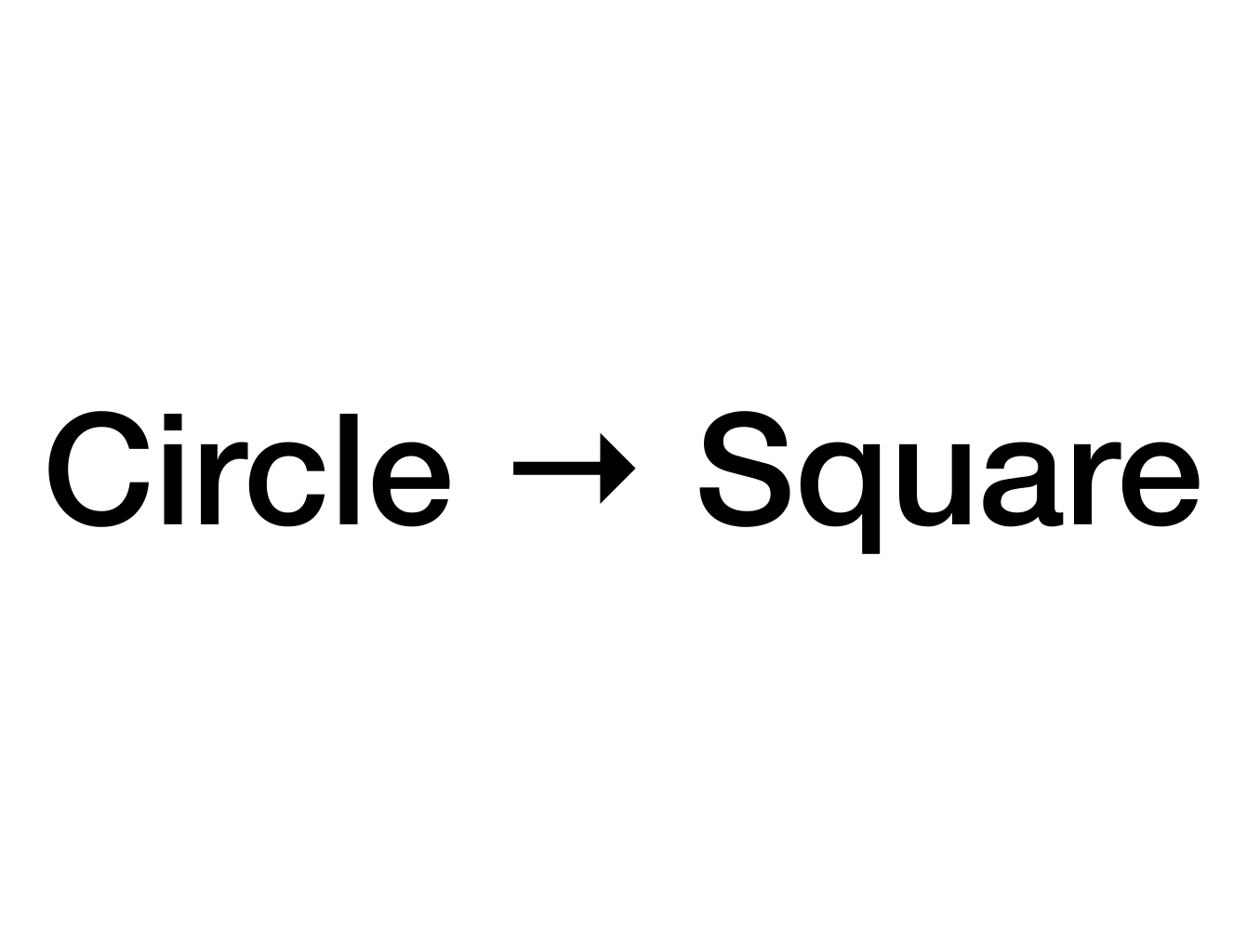}}
            {\includegraphics[width=.8\hsize]{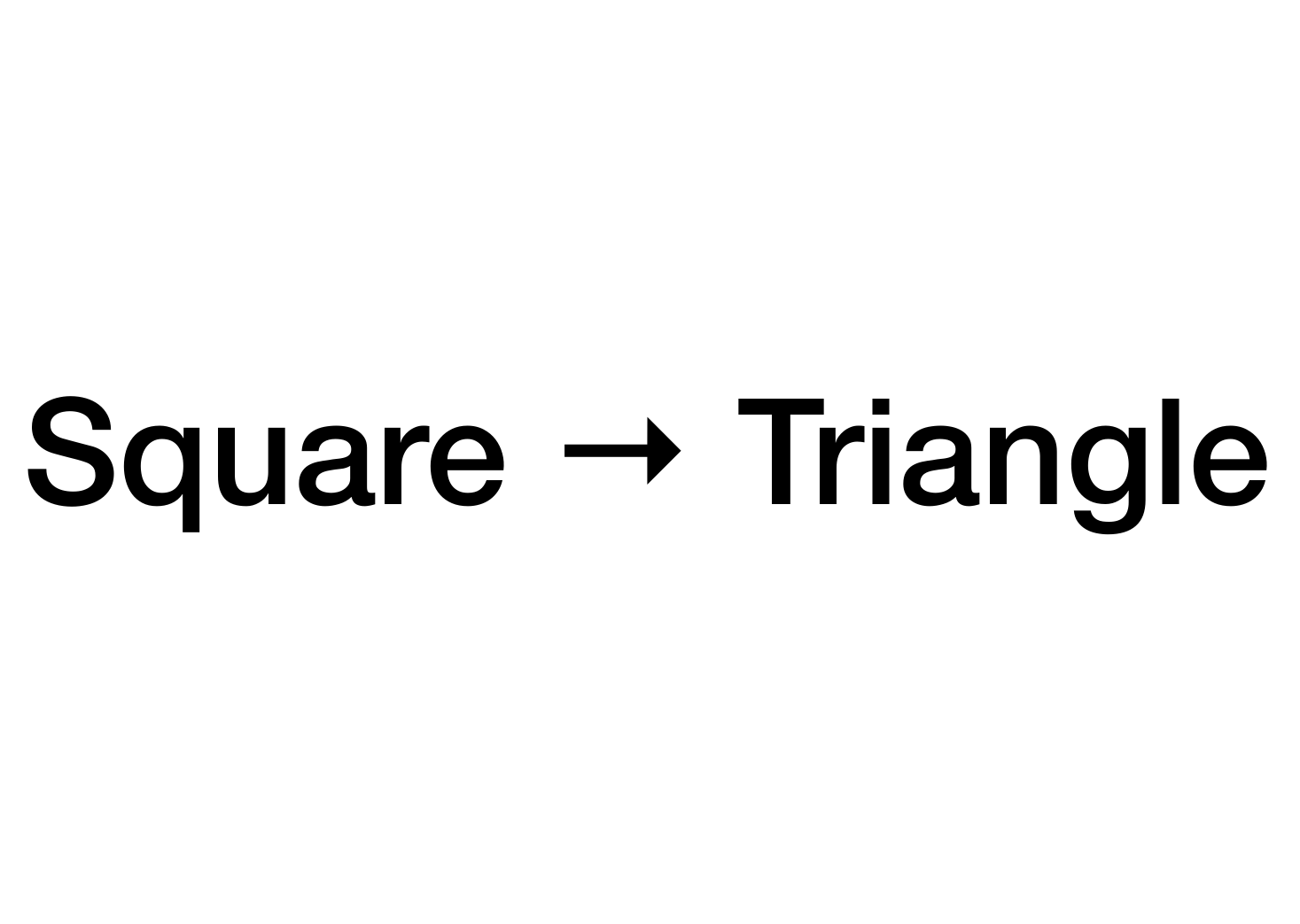}}
            {\includegraphics[width=.8\hsize]{circ-to-rec-text.pdf}}
        \end{center}
         }
         \caption{}
         \label{fig:adv-inputs}
     \end{subfigure}
     
    }
    }
        \caption{(a) Schematic for case study 2. (b) Shapes of the sign. (c) Example inputs in the bounding boxes that cause misclassification. The inputs are visualized using their values after preprocessing of the neural network. The left hand side is the ground truth label, and the right hand side is the misclassified label. }
\end{figure*}

\subsubsection{Neural Network}

We experiment with both a real-valued network and a binarized network for this case. For the real-valued network, we use a CNN with 2 convolutional layers (size 4$\times$4) with 16 filters, followed by a fully connected layer with 100 units. The output is a softmax layer with 3 entries, corresponding to the three shapes. For the binarized network, we use 2 convolutional layers (size 4$\times$4) with 16 and 32 filters respectively, followed by a fully connected layer with 100 units. The output is a linear layer with 3 output entries.

The distance measurements from LiDAR are preprocessed before feeding into both networks: first dividing them by \texttt{MAX\_RANGE} to scale to [0,1], then using 1 minus the scaled distances as inputs. This helps with the network training. We train the network using 50k points per class, and validating using 500 points per class. All the training and validation data are generated from points sampled randomly in the state space. We train the networks with cross-entropy loss, using \textit{Adam} \cite{kingma2014}. The real-valued network is trained with learning rate 0.01. The binarized network is trained with learning rate 0.0001.

\subsubsection{\textit{Tiler}}
The state tiles are constructed in each of the three shape subspaces ($c=0$, $c=1$, and $c=2$). We divide the $\theta$ dimension uniformly into $n_{\theta}$ intervals and the $d$ dimension uniformly in the inverse scale into $n_{d}$ intervals to obtain a grid with $n_{\theta}\cdot n_d$ cells per shape. Since each state tile lies within one of the shape subspaces, the ground truth bound for $\mathcal{S}_i$ is simply a singleton set containing the shape class of the subspace $\mathcal{S}_i$ lies in.

To compute the bounding box $\mathcal{B}_i$ for a given tile $\mathcal{S}_i$, we first find a lower bound and an upper bound on the distance of the object hit for each beam as the sensor position varies within that tile $\mathcal{S}_i$ (see Appendix \ref{sec:lidar-box} for details on the method to compute these bounds). We then extend this lower and upper bound by $5\sigma$, where $\sigma$ is the standard deviation of the Gaussian measurement noise. This way we have $P(x\in \mathcal{B}_i|x\sim g(s), s\in \mathcal{S}_i)\geq (P(|a| \leq 5\sigma | a \sim \mathcal{N}(0,\sigma^2)))^N > 0.999$, where $N=32\times 32$ is the input dimension. The factor 5 can be changed, depending on the required probabilistic guarantee. The resulting $\mathcal{B}_i$ is again an $l_{\infty}$-norm ball in the input space.

We apply the MILP method \cite{tjeng2019} (Appendix \ref{sec:milp}) to verify the real-valued network, and the SAT-based method \cite{eevbnn} (Appendix \ref{sec:eevbnn}) for the binarized network. For the real-valued network, we solve lower bounds on the differences between output scores, giving 2 optimization problems per tile. We then use Equation \ref{eqn:classification-tile-error} to decide whether the tile is verified to be correct or not. For the binarized network, we construct clauses using the output (according to Equation \ref{eqn:bnn-out}) and verify directly for each tile. We compare adaptive tiling with fixed tiling on both the percentage of regions verified and the verification speed.

\section{Experimental Results}
\label{sec:results}

\subsection{Case Study 1: Position Measurement from Road Scene}
We run \textit{Tiler} with a cell size of 0.1 (the side length of each cell in the $(\delta, \theta)$ grid is $a=b=0.1$). The step that takes the majority of time is the optimization solver. With parallelism, the optimization step takes about 15 hours running on 40 CPUs@3.00 GHz, solving $960000 \times 4$ MILP problems.

\subsubsection*{Global error bound}
We compute global error bounds by taking the maximum of $e_{\delta}^i$ and $e_{\theta}^i$ over all tiles. The global error bound for $\delta$ is 12.66, which is 15.8\% of the measurement range (80 length units for $\delta$); the global error bound for $\theta$ is $7.13\degree$ (5.94\% of the $120\degree$ measurement range). We therefore successfully verify the correctness of the network with these tolerances for all target inputs.

\subsubsection*{Error bound landscape}
We visualize the error bound landscape by plotting the error bounds of each tile as heatmaps over the $(\delta, \theta)$ space. Figures \ref{fig:offset-bound} and \ref{fig:angle-bound} present the results for $e_{\delta}^i$ and $e_{\theta}^i$ respectively. To further evaluate the distribution of the error bounds, we compute the percentage of the state space $\mathcal{S}$ (measured on the $(\delta, \theta)$ grid) that has error bounds below some threshold value. The percentage varying with threshold value can be viewed as a cumulative distribution. Figures \ref{fig:offset-upper-cumulative} and \ref{fig:angle-upper-cumulative} present the cumulative distributions of the error bounds. It can be seen that most of the state space can be verified with much lower error bounds, with only a small percentage of the regions having larger error bounds. This is especially the case for the offset measurement: 99\% of the state space is guaranteed to have error less than 2.65 (3.3\% of the measurement range), while the global error bound is 12.66 (15.8\%).

\begin{figure*}
  \centering
  \parbox{\figrasterwd}{
    \subcaptionbox{\label{fig:offset-bound}}{\includegraphics[width=.31\hsize]{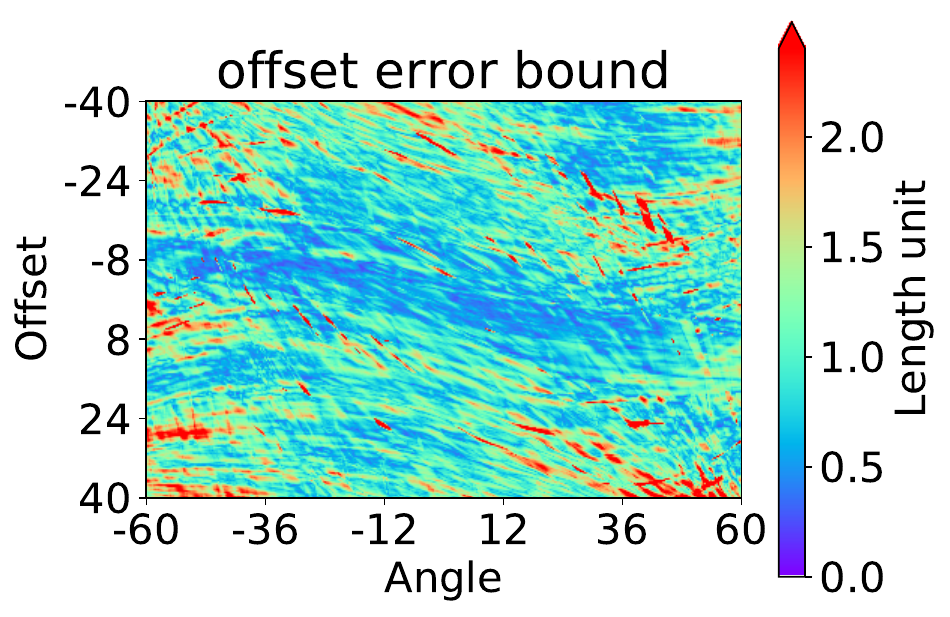}}
    \subcaptionbox{\label{fig:offset-gap}}{\includegraphics[width=.31\hsize]{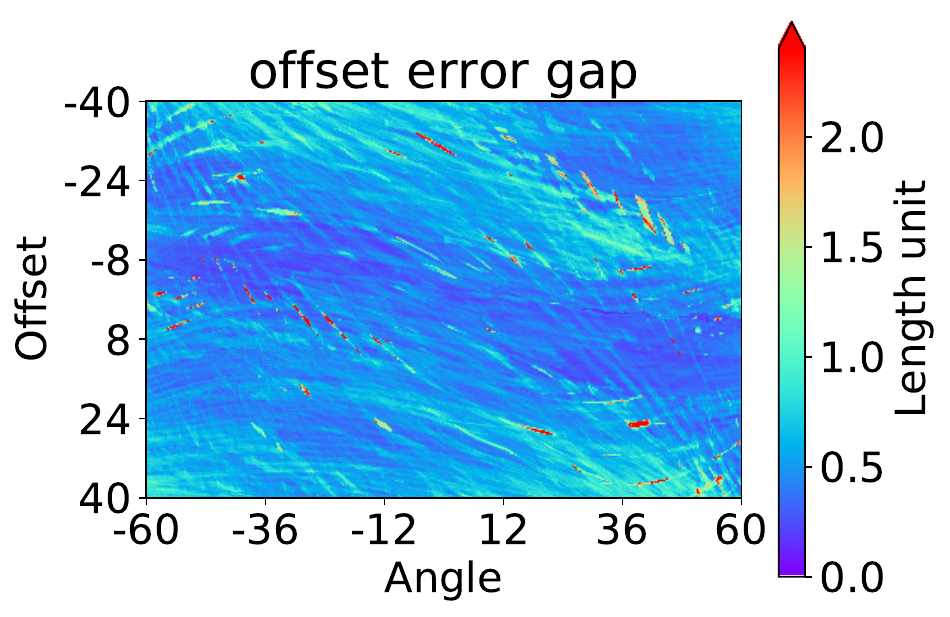}}
    \subcaptionbox{\label{fig:offset-upper-cumulative}}{\includegraphics[width=.29\hsize]{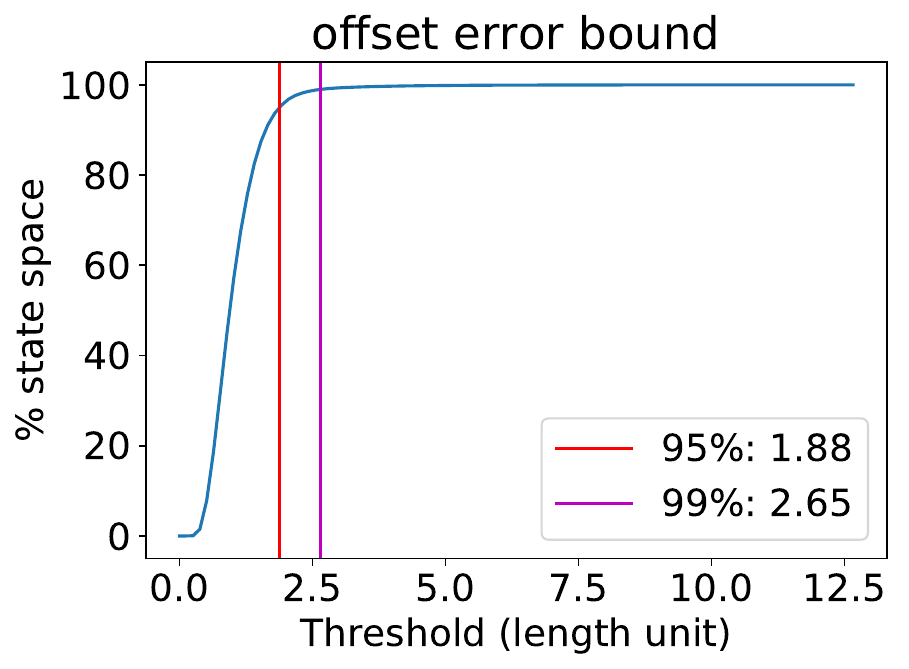}}
  }
  \parbox{\figrasterwd}{
    \subcaptionbox{\label{fig:angle-bound}}{\includegraphics[width=.31\hsize]{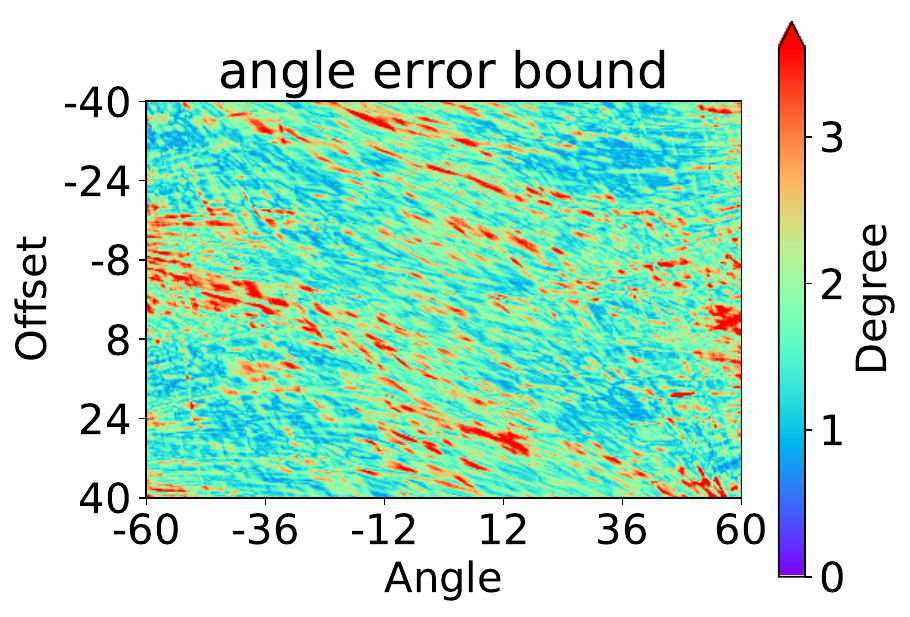}}
    \subcaptionbox{\label{fig:angle-gap}}{\includegraphics[width=.31\hsize]{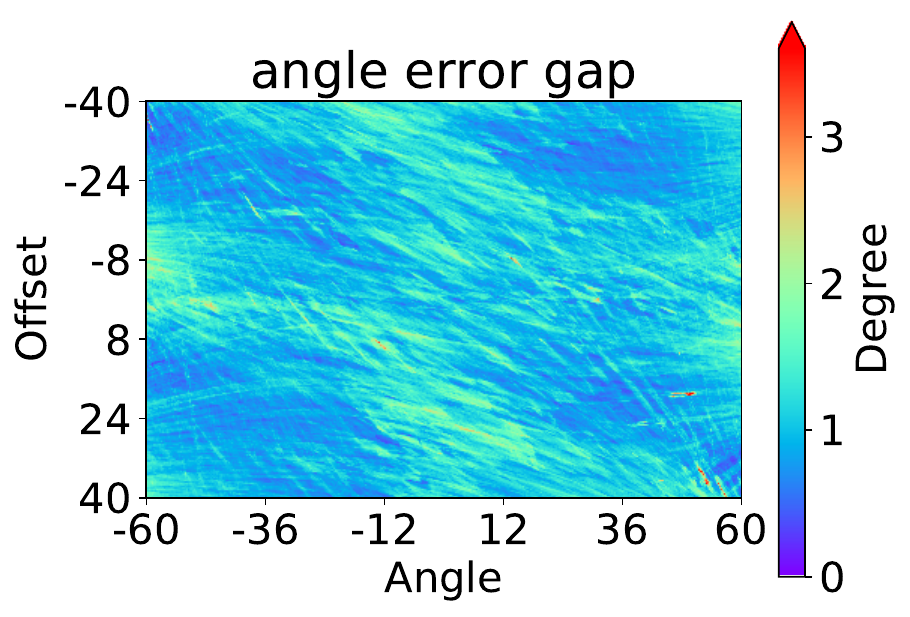}}
    \subcaptionbox{\label{fig:angle-upper-cumulative}}{\includegraphics[width=.29\hsize]{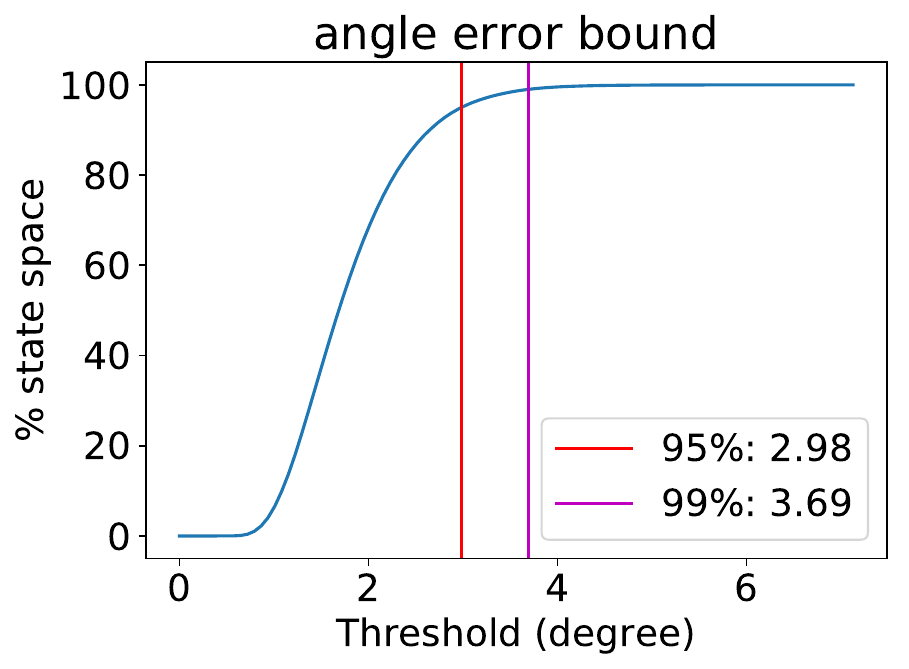}}
  }
  \vspace{-0.5em}
  \caption{(a,d) Heatmaps of the upper bounds on the maximum error of each tile over the offset-angle space. (b,e) Corresponding heatmaps after subtracting empirical estimates of the actual maximum error. (c,f) Percentage of the state space with error upper bounds below some threshold value (cumulative distribution).}
\end{figure*}

\subsubsection*{How tight are the error bounds}
A key question is how well the error bounds reflect the actual maximum error made by the neural network. To study the tightness of the error bounds, we compute empirical estimates of the maximum errors for each $\mathcal{S}_i$, denoted as $\Bar{e}_{\delta}^i$ and $\Bar{e}_{\theta}^i$. We sample multiple $(\delta, \theta)$ within each cell $\mathcal{S}_i$, generate input images for each $(\delta, \theta)$, then take the maximum over the errors of these points as the empirical estimate of the maximum error for $\mathcal{S}_i$. The sample points are drawn on a sub-grid within each cell, with sampling spacing 0.05. This estimate is a lower bound on the maximum error for $\mathcal{S}_i$, providing a reference for evaluating the tightness of the error upper bounds we get from \textit{Tiler}.

We take the maximum of $\Bar{e}_{\delta}^i$'s and $\Bar{e}_{\theta}^i$'s to get a lower bound estimate of the global maximum error. The lower bound estimate of the global maximum error for $\delta$ is 9.12 (11.4\% of the measurement range); for $\theta$ is $4.08\degree$ (3.4\% of the measurement range). We can see that the error bounds from \textit{Tiler} are close to the lower bound estimates derived from the observed errors that the network exhibits for specific inputs.

\hide{Similar to the heatmaps for the bounds $e_{\delta}^i$ and $e_{\theta}^i$, we plot the heatmaps for the error estimates $\Bar{e}_{\delta}^i$ and $\Bar{e}_{\theta}^i$, shown in Figure \ref{fig:offset-estimate} and \ref{fig:angle-estimate}. It can be seen that the landscape of the bounds obtained from \textit{Tiler} is highly similar to the landscape of the error estimates. This suggests that the error bounds from \textit{Tiler} provide a good reflection on the actual maximum error made by the network. }

Having visualized the heatmaps for $e_{\delta}^i$ and $e_{\theta}^i$, we subtract the error estimates $\Bar{e}_{\delta}^i$ and $\Bar{e}_{\theta}^i$ and plot the heatmaps for the resulting gaps in Figures \ref{fig:offset-gap} and \ref{fig:angle-gap}. We can see that most of the regions that have large error bounds are due to the fact that the network itself has large errors there. By computing the cumulative distributions of these gaps between bounds and estimates, we found that for angle measurement, 99\% of the state space has error gap below $1.9\degree$ (1.6\% of measurement range); and for offset measurement, 99\% of the state space has error gap below 1.41 length units (1.8\%). The gaps indicate the maximum possible improvements on the error bounds.

\hide{\begin{figure}
     \centering
     \begin{subfigure}[b]{0.49\textwidth}
         \centering
         \includegraphics[width=\textwidth]{offset_gap_cumulative.png}
         \caption{}
         \label{fig:offset-gap-cumulative}
     \end{subfigure}
     \hfill
     \begin{subfigure}[b]{0.49\textwidth}
         \centering
         \includegraphics[width=\textwidth]{angle_gap_cumulative.png}
         \caption{}
         \label{fig:angle-gap-cumulative}
     \end{subfigure}
    \caption{Cumulative distributions of the gap between bounds and estimates on the maximum error of each tile. (a) is for offset measurement, (b) is for angle measurement.}
    \label{fig:error-gap}
\end{figure}}

\subsubsection*{Contributing factors to the gap}
There are two main contributing factors to the gap between the bound from \textit{Tiler} and the actual maximum error. The first factor is that we use interval arithmetic to compute the error bound in \textit{Tiler}: \textit{Tiler} takes the maximum distance between the range of possible ground truths and the range of possible network outputs as the bound. For a perfect network, this still gives a bound equal to the range of ground truth, instead of zero. The second factor is the extra space included in $\mathcal{B}_i$ that is not on the tile $\mathcal{X}_i$. This results in a larger range on network output being used to calculate error bounds, which makes the error bounds larger.

\subsubsection*{Effect of tile size}
Both of the factors described above are affected by the tile size. We run \textit{Tiler} with a sequence of cell sizes (0.05, 0.1, 0.2, 0.4, 0.8) for the $(\delta, \theta)$ grid. Figure \ref{fig:99-percentiles-grid-size} shows how the 99 percentiles of the error upper bounds and the gap between error bounds and estimates vary with cell size. As tile size gets finer, \textit{Tiler} provides better error bounds, and the tightness of bounds improves.

These results show that \textit{Tiler} gives better error bounds with finer tile sizes. Meanwhile, reducing tile sizes also increases the total number of tiles and the number of optimization problems to solve. Figure \ref{fig:time-grid-size} shows how the total solving time varies with cell size. For cell sizes smaller than 0.2, we indeed see that smaller tile sizes lead to longer total time. For cell sizes larger than 0.2, total solving time increases with cell size instead. The reason is that each optimization problem becomes harder to solve as the tile becomes large. Specifically, the approach we adopt \cite{tjeng2019} relies on presolving ReLU stability to improve speed. The number of unstable ReLUs will increase drastically as the cell size becomes large, which makes the solving slower.

\begin{figure*}
     \centering
      \parbox{\figrasterwd}{
      \parbox{.31\figrasterwd}{%
     \begin{subfigure}[b]{0.33\textwidth}
         \centering
         \includegraphics[width=\textwidth]{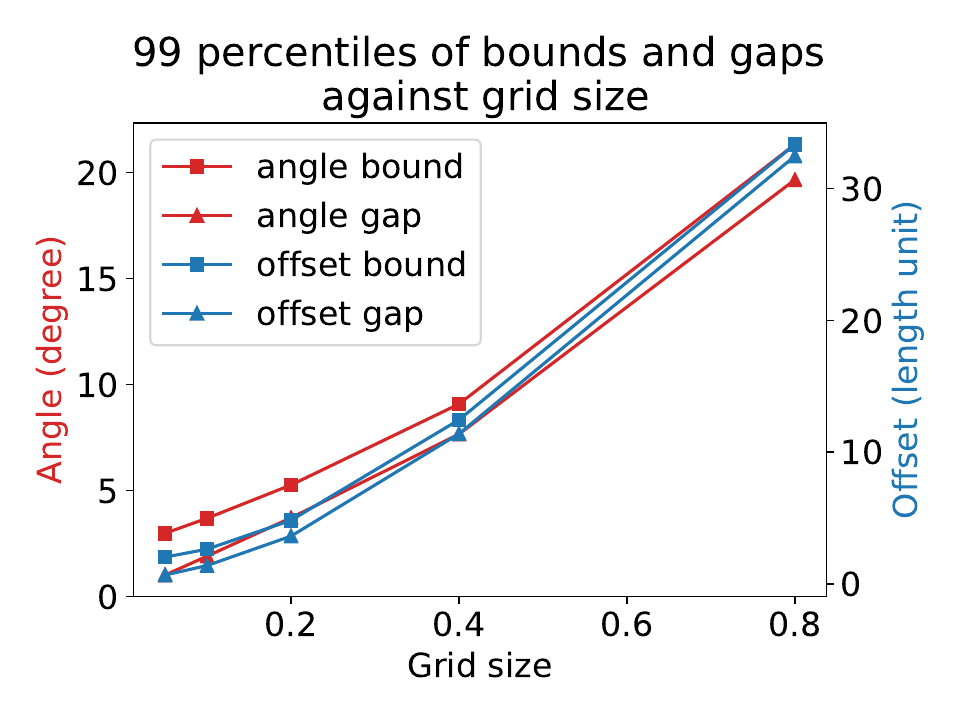}
         \vspace{-1em}
         \caption{}
         \label{fig:99-percentiles-grid-size}
     \end{subfigure}
     }
     \parbox{.30\figrasterwd}{%
     \begin{subfigure}[b]{0.32\textwidth}
         \centering
         \includegraphics[width=\textwidth]{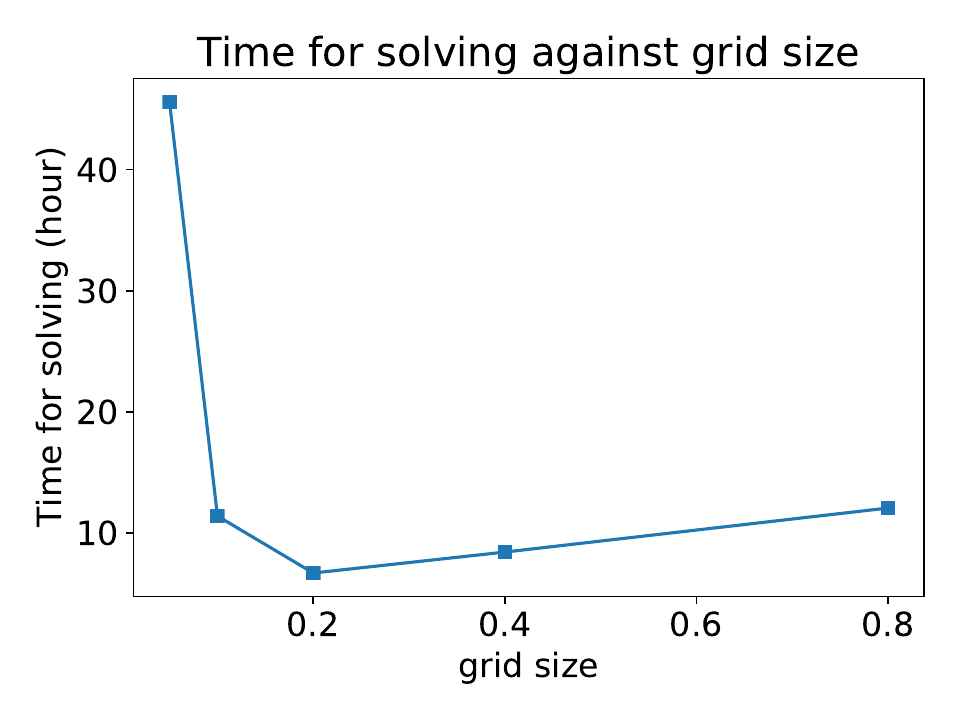}
         \vspace{-1em}
         \caption{}
         \label{fig:time-grid-size}
     \end{subfigure}
    }
    \parbox{.16\figrasterwd}{%
    \begin{center}
        {\includegraphics[width=.5\hsize]{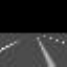}}
        \subcaptionbox{\label{fig:in}}{\includegraphics[width=.5\hsize]{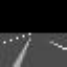}}
    \end{center}
    }\hskip-2em
    \parbox{.16\figrasterwd}{%
    \begin{center}
        {\includegraphics[width=.5\hsize]{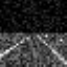}}
        \subcaptionbox{\label{fig:noise}}{\includegraphics[width=.5\hsize]{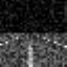}}
    \end{center}
    }\hskip-2em
    \parbox{.16\figrasterwd}{%
    \begin{center}
        {\includegraphics[width=.5\hsize]{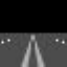}}
        \subcaptionbox{\label{fig:new}}{\includegraphics[width=.5\hsize]{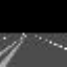}}
    \end{center}
    }}
        \caption{(a) 99 percentiles of the error upper bounds and the gaps between upper and lower bounds against cell size. (b) Total solving time against cell size. (c) Example legal inputs. (d) Example corrupted inputs. (e) Example inputs from new scene. }
\end{figure*}

\subsubsection*{Adaptive tiling}
We next consider an adaptive tiling approach (as described in Section \ref{sec:adaptive}). We consider an error threshold of 2.65 in $\delta$ and $3.69\degree$ in $\theta$ under which the neural network is considered correct. These threshold values only serve as an example and can be chosen differently (these values are taken from the 99 percentiles of error bounds in Figure \ref{fig:offset-upper-cumulative} and \ref{fig:angle-upper-cumulative}). We set the timeout for MILP solvers to be 5 seconds. We start with a cell size of 0.2, and stop further dividing the tiles when the cell sizes reaches a minimum threshold of 0.05.

We compare adaptive tiling with fixed tiling in terms of percentage of state space verified (max error less than the thresholds) and verification time. Figure \ref{fig:tradeoff-perc} shows that adaptive tiling gives the highest percentage of regions verified compared with fixed tiling strategies. Figure \ref{fig:tradeoff-time} shows that adaptive tiling uses less time than all fixed tiling strategies except for the cases for cell sizes 0.2 and 0.4, but fixed tiling with cell size 0.2 only verifies 85.2\% of state space (and 13.5\% for cell size 0.4) compared with 99.7\% for adaptive tiling. To achieve the same percentage verified, fixed tiling (with cell size 0.05) needs 45.5 hours of verification time compared with 9.05 hours for adaptive tiling. We further plot the distribution of final tile sizes in adaptive tiling in Figure \ref{fig:finalsize}. We can see that the majority of the state space (over 80\%) can be verified with the initial large tile size, so \textit{Tiler} only needs to put more work on the remaining regions. This is the reason why adaptive tiling is more efficient.

\begin{figure*}
     \centering
      \parbox{\figrasterwd}{
      \parbox{.32\figrasterwd}{%
     \begin{subfigure}[b]{0.33\textwidth}
         \centering
         \includegraphics[width=\textwidth]{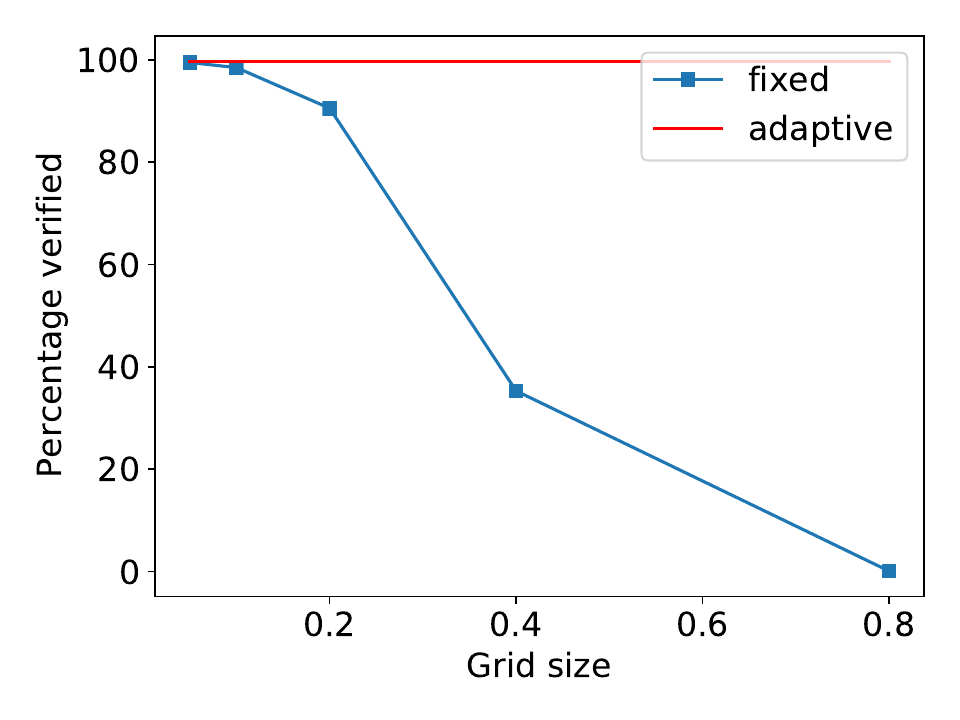}
         \vspace{-1em}
         \caption{}
         \label{fig:tradeoff-perc}
     \end{subfigure}
     }
     \parbox{.32\figrasterwd}{%
     \begin{subfigure}[b]{0.32\textwidth}
         \centering
         \includegraphics[width=\textwidth]{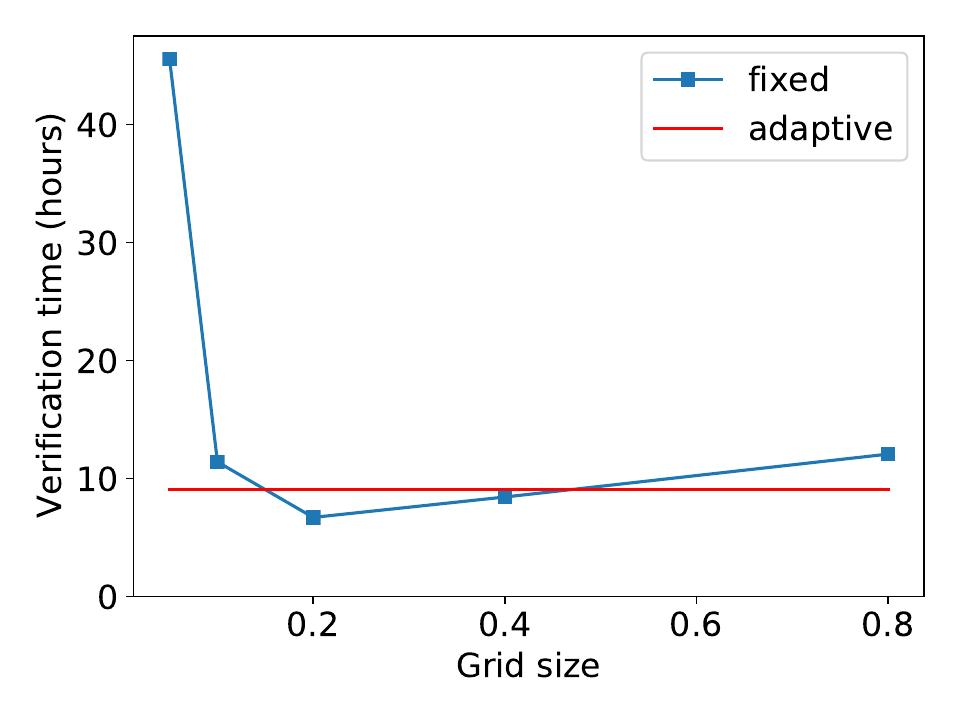}
         \vspace{-1em}
         \caption{}
         \label{fig:tradeoff-time}
     \end{subfigure}
    }
    \parbox{.32\figrasterwd}{%
    \begin{subfigure}[b]{0.32\textwidth}
        \centering
        \includegraphics[width=\textwidth]{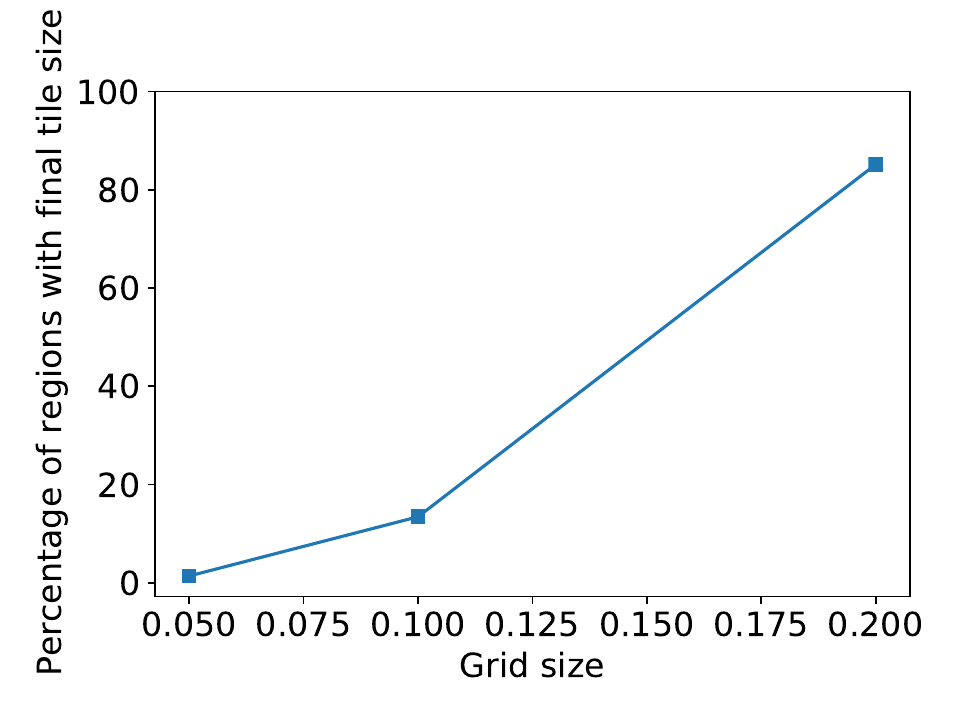}
        \vspace{-1em}
        \caption{}
        \label{fig:finalsize}
    \end{subfigure}
   }
    }
        \caption{(a) Percentage of region verified, fixed v.s. adaptive tiling. (b) Total verification time, fixed v.s. adaptive tiling. (c) Distribution of final tile sizes in adaptive tiling.}
\end{figure*}

\subsubsection*{Detecting illegal inputs}
We implement the input detector by checking if the new input $x^*$ is contained in any of the bounding boxes $\mathcal{B}_i$. We test the detector with 3 types of inputs: 1) legal inputs, generated from the state space through the imaging process; 2) corrupted inputs, obtained by applying i.i.d uniformly distributed per-pixel perturbation to legal inputs; 3) inputs from a new scene, where the road is wider and there is a double centerline. Figure \ref{fig:in} to \ref{fig:new} show some example images for each type. We randomly generated 500 images for each type. Our detector is able to flag all inputs from type 1 as legal, and all inputs from type 2 and 3 as illegal. On average, naive search (over all $\mathcal{B}_i$) takes 1.04s per input, while prediction-guided search takes 0.04s per input. So the prediction-guided search gives a 26$\times$ speedup without any compromise in functionality.

\subsection{Case Study 2: Shape Classification from LiDAR Sensing}

\subsubsection*{Verification results}

We divide the $\theta$ dimension into 90 intervals and the $d$ dimension into 60 intervals to obtain a grid with 5400 cells per shape. We plot the verification results as heatmaps over the state space. Figure \ref{fig:lidar-heatmaps} (top row) shows the results for MILP-based verification of real-valued network. The results for SAT-based verification of binarized network is similar (Figure \ref{fig:lidar-heatmaps-bnn}). We are able to verify the correctness of the network over the majority of the state space. In particular, we verify that the network is always correct when the shape of the sign is triangle.

We also run a finer tiling with half the cell sizes in both $d$ and $\theta$, giving 21600 cells per shape. Figure \ref{fig:lidar-heatmaps} (bottom row) shows the verification results. By reducing the tile sizes, we can verify more regions in the state space.

\begin{figure*}
     \centering
      \parbox{\figrasterwd}{
    \parbox{.3\figrasterwd}{%
    \begin{center}
        {\includegraphics[width=.8\hsize]{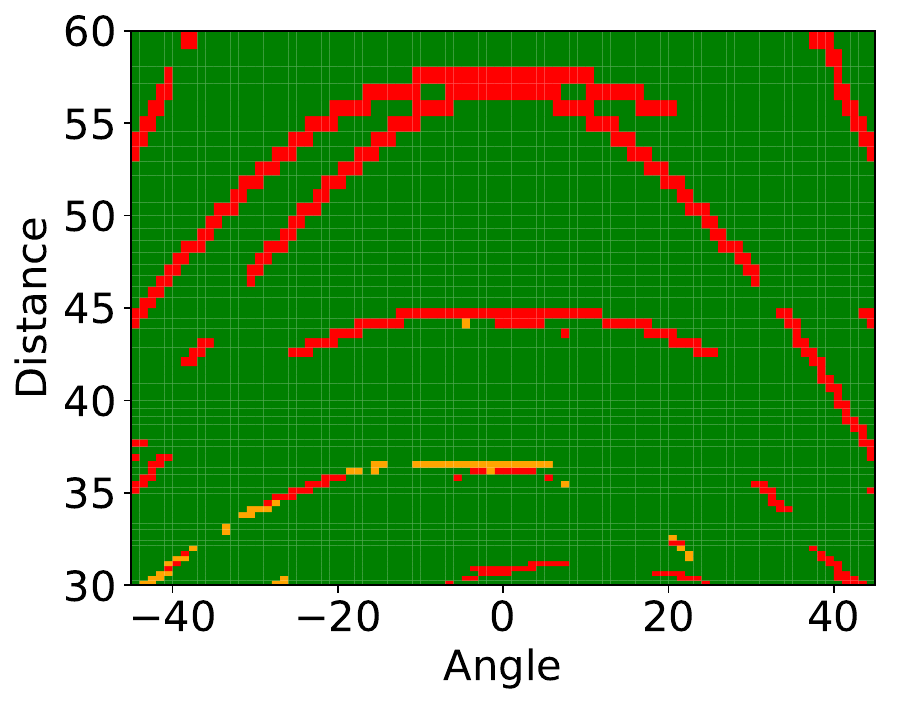}}
        \subcaptionbox*{Square}{\includegraphics[width=.8\hsize]{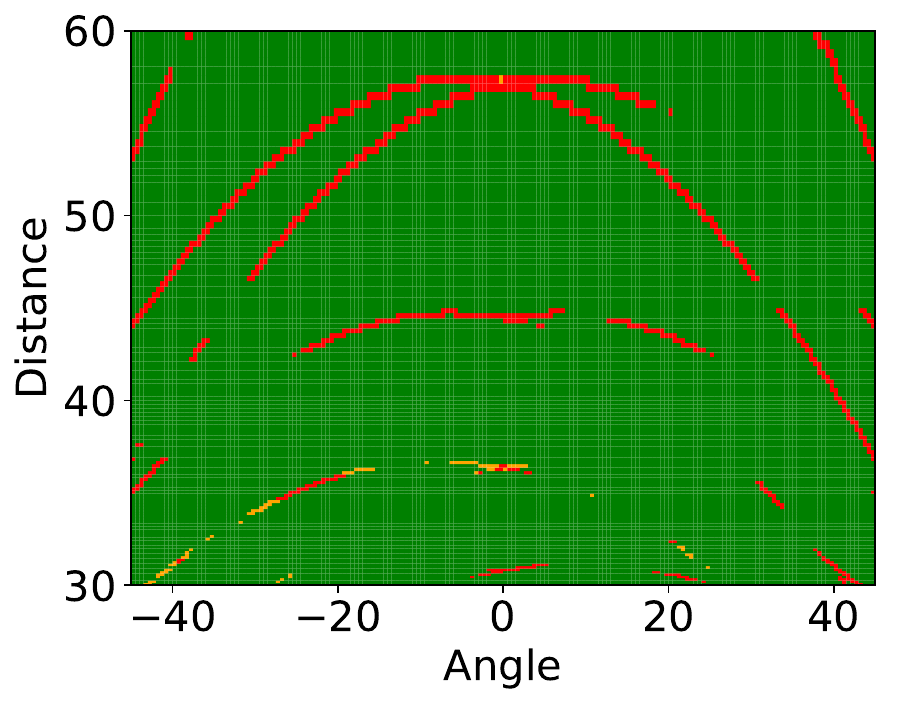}}
    \end{center}
    }
    \parbox{.3\figrasterwd}{%
    \begin{center}
        {\includegraphics[width=.8\hsize]{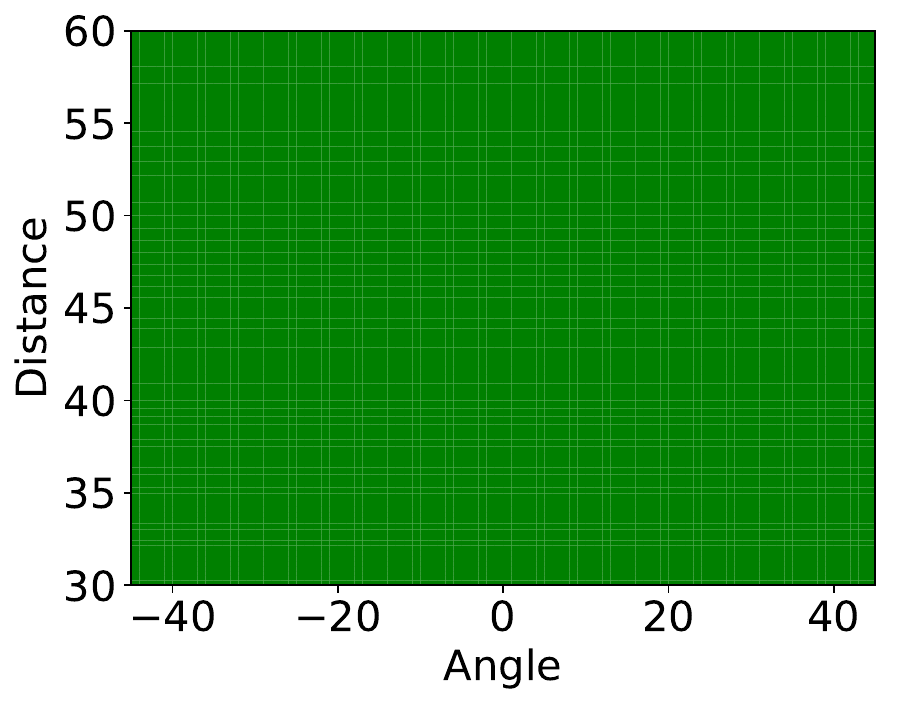}}
        \subcaptionbox*{Triangle}{\includegraphics[width=.8\hsize]{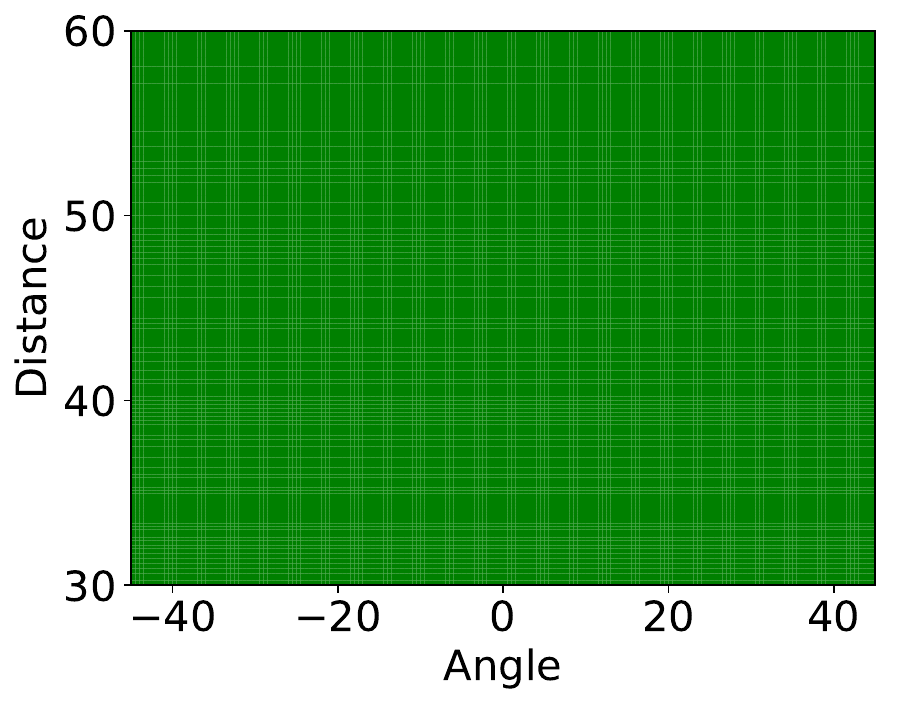}}
    \end{center}
    }
    \parbox{.3\figrasterwd}{%
    \begin{center}
        {\includegraphics[width=.8\hsize]{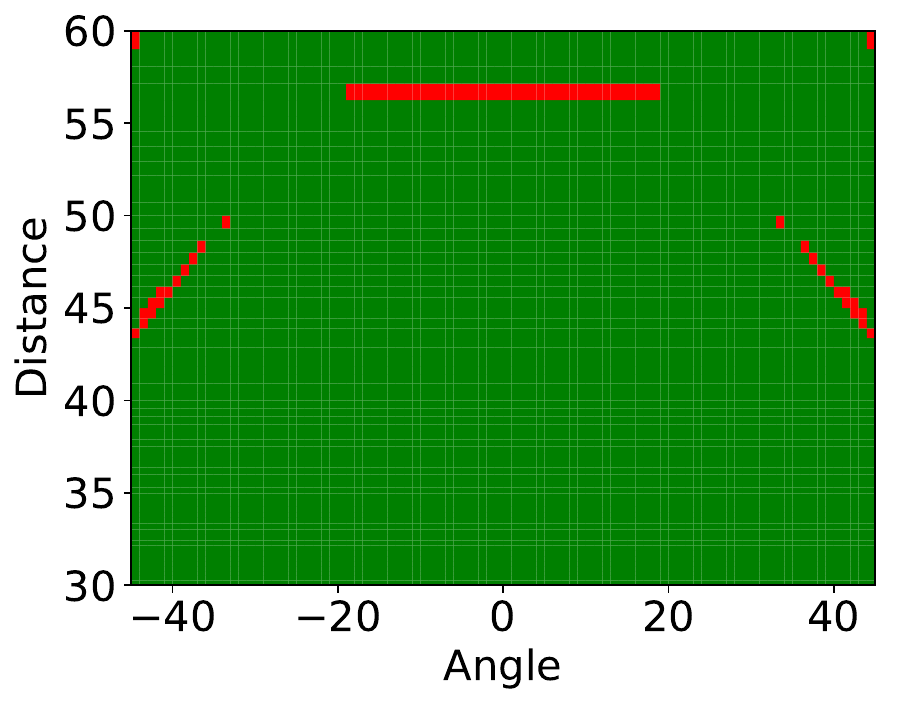}}
        \subcaptionbox*{Circle}{\includegraphics[width=.8\hsize]{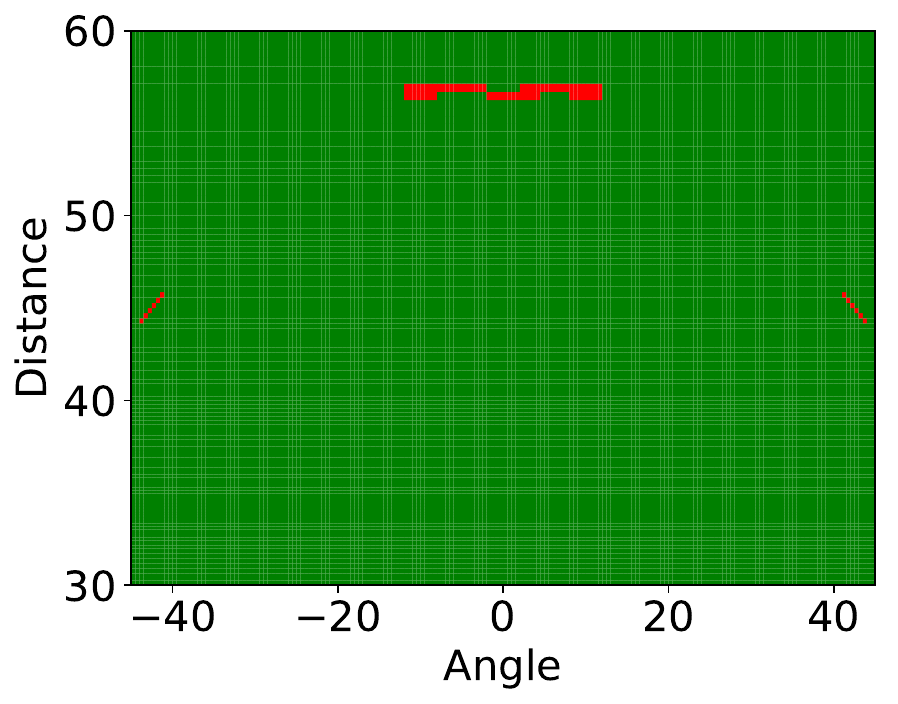}}
    \end{center}
    }}
        \caption{Verification results over the state space for the LiDAR case study, for real-valued network. Top row: tiling with a grid of 90$\times$60 cells; bottom row: a grid of 180$\times$120 cells. Green cells indicate that the tile is verified ($e_i=0$); red cells indicate that we cannot verify that tile ($e_i=1$); orange cells indicate that the solver exceeds time limit while solving for that tile.}
        \label{fig:lidar-heatmaps}
        \vspace{-1em}
\end{figure*}

\begin{figure*}
     \centering
      \parbox{\figrasterwd}{
    \parbox{.3\figrasterwd}{%
    \begin{center}
        {\includegraphics[width=.8\hsize]{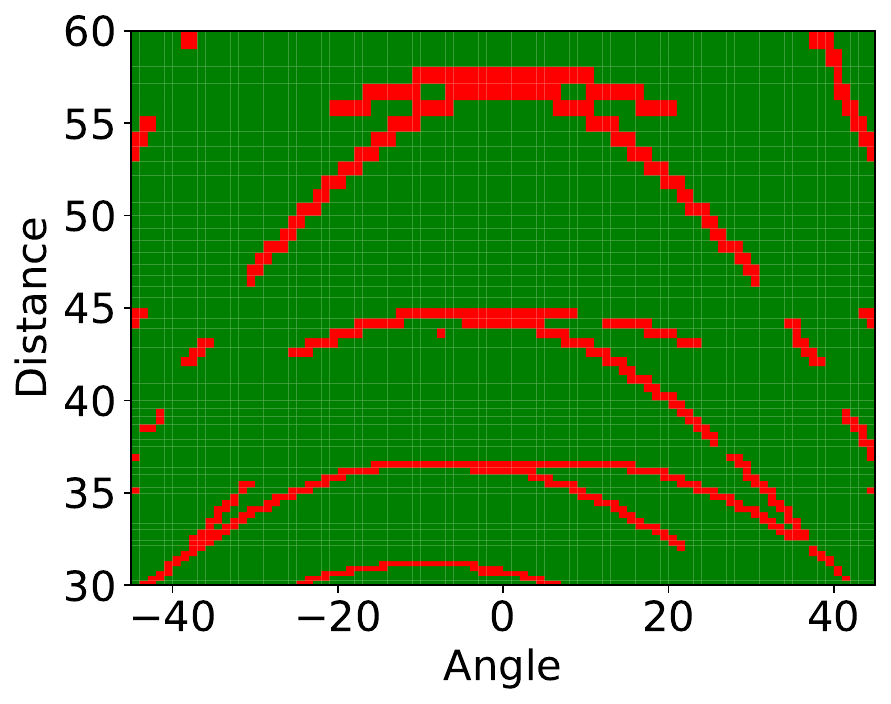}}
        \subcaptionbox*{Square}{\includegraphics[width=.8\hsize]{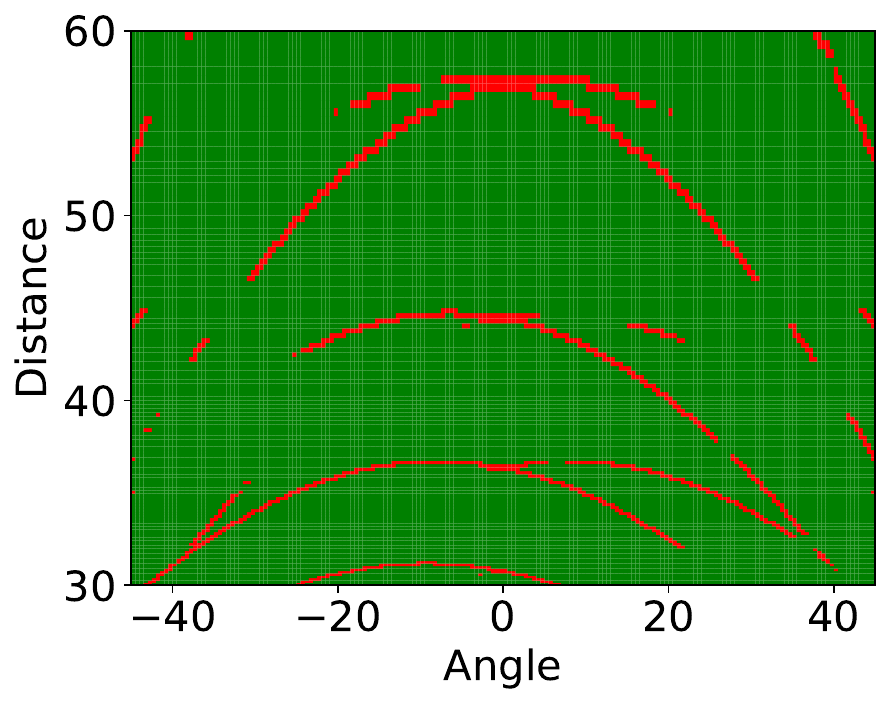}}
    \end{center}
    }
    \parbox{.3\figrasterwd}{%
    \begin{center}
        {\includegraphics[width=.8\hsize]{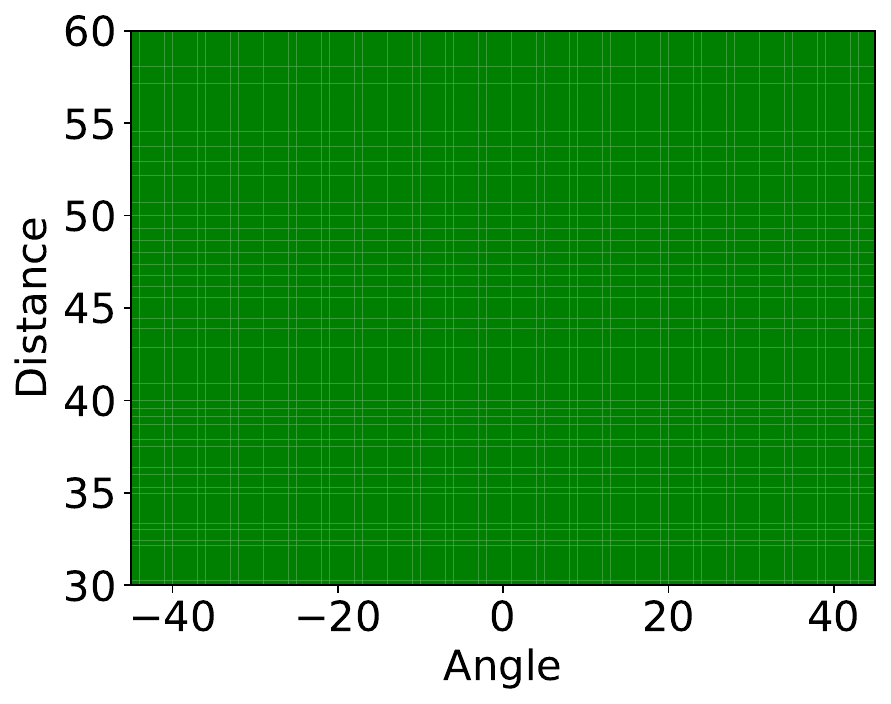}}
        \subcaptionbox*{Triangle}{\includegraphics[width=.8\hsize]{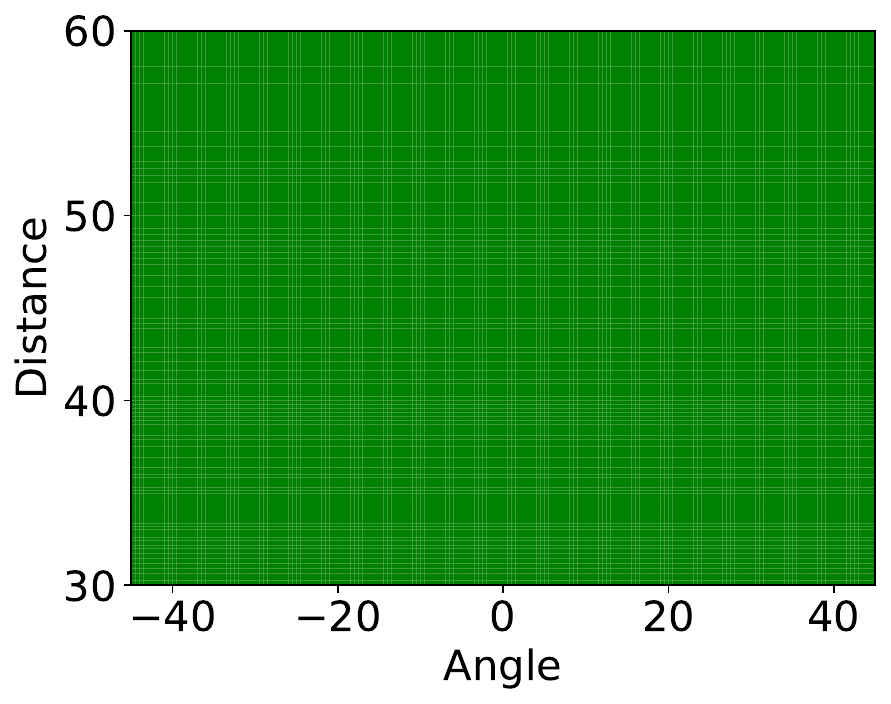}}
    \end{center}
    }
    \parbox{.3\figrasterwd}{%
    \begin{center}
        {\includegraphics[width=.8\hsize]{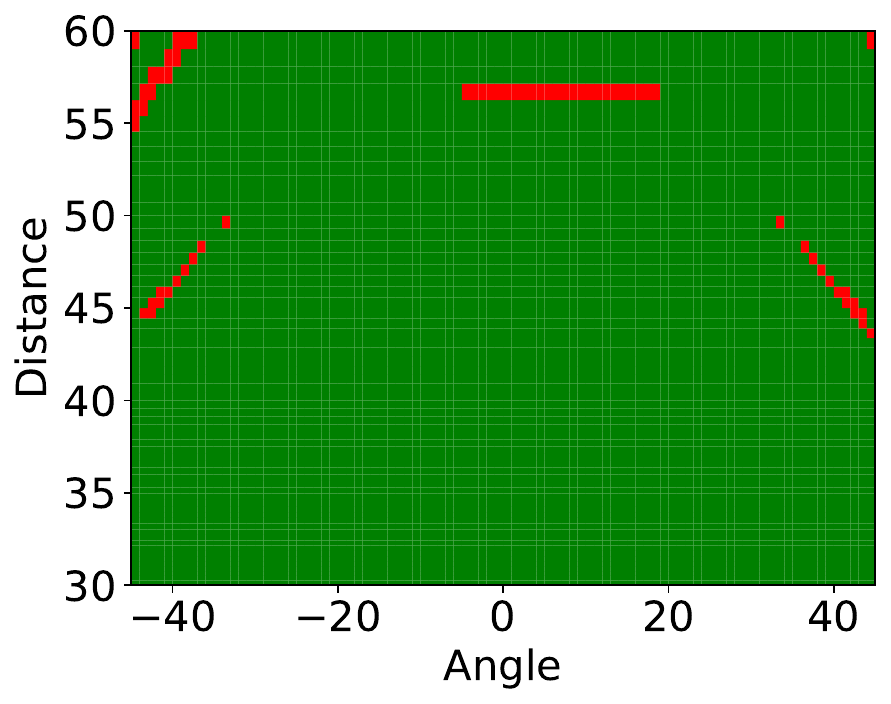}}
        \subcaptionbox*{Circle}{\includegraphics[width=.8\hsize]{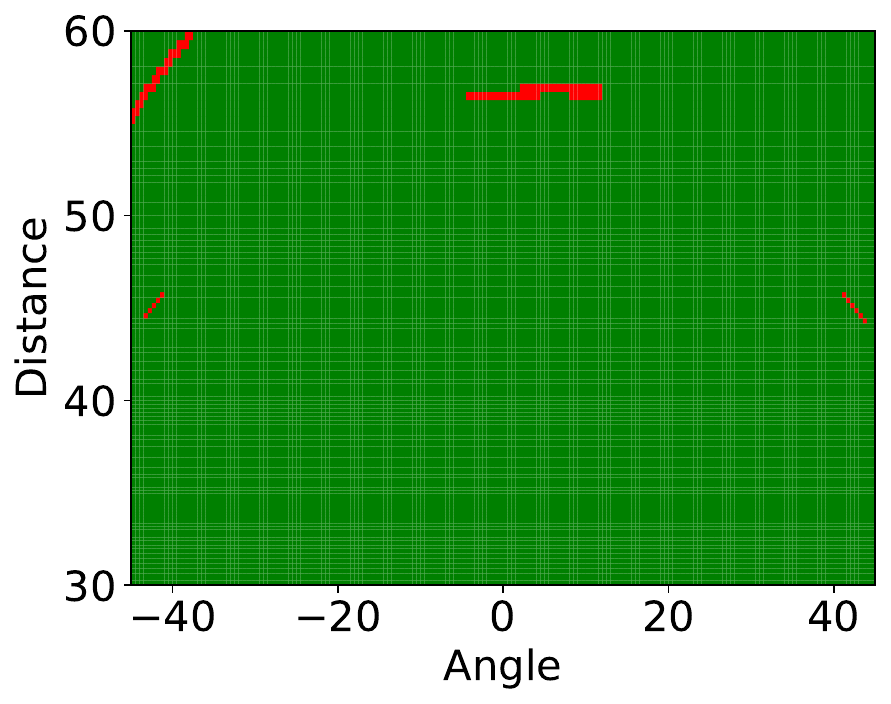}}
    \end{center}
    }}
        \caption{Verification results over the state space for the LiDAR case study, for binarized network. The top row corresponds to tiling with a grid of 90$\times$60 cells, the bottom row corresponds to a grid of 180$\times$120 cells.}
        \label{fig:lidar-heatmaps-bnn}
        \vskip-2em
\end{figure*}

For the tiles that we are unable to verify correctness (red squares in the heatmaps), there are inputs within those bounding boxes on which the network will predict a different class. We inspect several of such tiles to see the inputs that cause problems. Figure \ref{fig:adv-inputs} shows a few examples. In some of the cases (top two in figure), the `misclassified' inputs actually do not look like coming from the ground truth class. This is because of the extra space included in $\mathcal{B}_i$ -- it includes inputs that are reasonably different from target inputs. Such cases will be reduced as the tile size becomes smaller, since the extra spaces in $\mathcal{B}_i$'s will be shrunk. We have indeed observed such phenomenon, as we can verify more regions when the tile size becomes smaller. In some other cases (bottom example), however, the misclassified inputs are perceptually similar to the ground truth class. Yet the network predicts a different class. This reveals that the network is potentially not very robust on inputs around these points. In this sense, our framework provides a way to systematically find regions of the target input space where the network is potentially vulnerable.

\subsubsection*{Adaptive tiling and performance comparisons}
Here we compare the performance between MILP-based verification and SAT-based verification. We compare based on the percentage of region verified and verification time. We also compare adaptive tiling with fixed tiling. For adaptive tiling, we start with a 45$\times$30 grid. We set a timeout of 5 seconds for the solvers and a minimum tile size of 0.5$\degree$ in $\theta$ dimension. We stop further dividing tiles until either the tile is verified or the minimum tile size is reached. Table \ref{tab:lidar-performance} presents the performance comparisons. We can see that SAT-based BNN verification is 200--500x faster than MILP-based real-valued network verification. This speedup is consistent with the results from \cite{eevbnn}, where they test one MLP network and two CNNs on MNIST and CIFAR10. We attribute this speedup to the performance enhancement techniques described in Appendix \ref{sec:eevbnn}. Moreover, adaptive tiling achieves the highest percentage of regions verified with a ~5x speedup compared with fixed tiling that gives the same percentage verified.

\begin{table}[]
    \centering
    
    \begin{tabular}{cccccc}
    \toprule
    Network & \#Tiles & Time (s) & \% verified \\
    \midrule
    {\multirow{3}{*}{Real-valued}} & 5400 & 17124 & 96.9 \\
                                   & 21600 & 35599 & \textbf{98.6} \\
                                   & adaptive & \textbf{7563} & \textbf{98.6} \\ % 1566 + 2749 + 3248 + 5116
    \midrule
    {\multirow{3}{*}{BNN}} & 5400 & 34.9 & 96.3 \\
                           & 21600 & 180.9 & \textbf{98.3} \\
                           & adaptive & \textbf{34.8} & \textbf{98.3} \\ % 11.8 + 8.1 + 14.9 + 27.0
    \bottomrule
    \end{tabular}
    \caption{Comparison between real-valued network v.s. BNN verification, adaptive tiling v.s. fixed tiling.}
    \label{tab:lidar-performance}
\end{table}

\section{Conclusion}

In this paper, we present a framework for verifying correctness of neural networks for perception tasks. We introduce the state space of the world and the observation process that maps from states to neural network inputs as specification for the neural network. We present \textit{Tiler}, an algorithm for correctness verification in this framework. This technique allows us to verify that a neural network for perception tasks produces a {\em correct} output within a specified tolerance for {\em every} input in the target input space. This technique also allows us to detect inputs that the neural network is not designed to work on, so that we can flag inputs that have no verified correctness guarantee. Results from the case studies evaluate how well the approach works for a state space characterized by several attributes and a camera imaging or LiDAR measurement observation process. We anticipate that the technique will also work well for other problems that have a low dimensional state space (but potentially a high dimensional input space). For higher dimensional state spaces, the framework makes it possible to systematically target specific regions of the input space to verify. Potential applications include targeted verification, directed testing, and the identification of illegal inputs for which the network is not expected to work on.

%
% ---- Bibliography ----
%
% BibTeX users should specify bibliography style 'splncs04'.
% References will then be sorted and formatted in the correct style.
%
\bibliographystyle{splncs04}
\bibliography{references}

\newpage

\appendix
\input{appendix}

\end{document}

%% file: appendix.tex
\section{Proofs}
\label{sec:proof-append}
\setcounter{theorem}{0}
\begin{theorem}[Soundness of local error bound for regression]
\label{theorem-regression-local}
Given that Condition 4.1, 4.2(a), 4.3, and 4.4(a) are satisfied, then $\forall x \in \tilde{\mathcal{X}}$, $e(x) \leq e_{\text{local}}(x)$, where $e(x)$ is defined in Equation \ref{eqn:error} and $e_{\text{local}}(x)$ is computed from Equation \ref{eqn:regression-tile-error} and \ref{eqn:regression-local}.
\end{theorem}

\begin{proof}
For any $x \in \tilde{\mathcal{X}}$, we have
\begin{align*}
    e(x) &= \max_{y\in \hat{f}(x)} d(f(x), y) \\
    &= \max_{y\in \hat{f}(x)} |f(x) - y| \\
    &= \max_{y\in \hat{f}(x)} \{ \max (f(x) - y, y - f(x))\}
\end{align*}
Condition 4.1 and 4.2(a) guarantees that for any $x\in \tilde{\mathcal{X}}$ and $y \in \hat{f}(x)$, we can find a tile $\mathcal{X}_i$ such that $ x\in \mathcal{X}_i$ and $l_i \leq y \leq u_i$. Let $t(y, x)$ be a function that gives such a tile $\mathcal{X}_{t(y, x)}$ for a given $x$ and $y \in \hat{f}(x)$. Then
\begin{align*}
    e(x) &\leq \max_{y\in \hat{f}(x)} \{ \max (f(x) - l_{t(y, x)}, u_{t(y, x)} - f(x))\}.
\end{align*}
Since $x \in \mathcal{X}_{t(y, x)}$ and $\mathcal{X}_{t(y, x)} \subseteq \mathcal{B}_{t(y, x)}$ (Condition 4.3), $x \in \mathcal{B}_{t(y, x)}$. By Condition 4.4(a),
\begin{align*}
    l'_{t(y, x)} \leq f(x) \leq u'_{t(y, x)}, \forall y \in \hat{f}(x).
\end{align*}
This gives
\begin{align*}
    e(x) &\leq \max_{y\in \hat{f}(x)} \{ \max (u'_{t(y, x)} - l_{t(y, x)}, u_{t(y, x)} - l'_{t(y, x)})\} \\
    &= \max_{y\in \hat{f}(x)} e_{t(y, x)}
\end{align*}
Since $x \in \mathcal{B}_{t(y, x)}$ for all $y \in \hat{f}(x)$, we have $\{t(y, x) | y \in \hat{f}(x) \} \subseteq \{i | x\in \mathcal{B}_i\}$, which gives
\begin{align*}
    e(x) &\leq \max_{\{t(y, x) | y \in \hat{f}(x) \}} e_{t(y, x)} \\
    &\leq \max_{\{i | x\in \mathcal{B}_i\}} e_i \\
    &= e_{\text{local}}(x)
\end{align*}

\end{proof}

\begin{theorem}[Soundness of global error bound for regression]
\label{theorem-regression-global}
Given that Condition 4.1, 4.2(a), 4.3, and 4.4(a) are satisfied, then $\forall x \in \tilde{\mathcal{X}}$, $e(x) \leq e_{\text{global}}$, where $e(x)$ is defined in Equation \ref{eqn:error} and $e_{\text{global}}$ is computed from Equation \ref{eqn:regression-tile-error} and \ref{eqn:regression-global}.
\end{theorem}

\begin{proof}
By Theorem \ref{theorem-regression-local}, we have $\forall x \in \tilde{\mathcal{X}}$,
\begin{align*}
    e(x) &\leq \max_{\{i | x\in \mathcal{B}_i\}} e_i \\
    &\leq \max_{i} e_i \\
    &= e_{\text{global}}
\end{align*}
\end{proof}

\section{\textit{Tiler} for classification}
\label{sec:classification}

\noindent \textbf{Step 1:} (tiling the space) is the same as regression.
\\

\noindent \textbf{Step 2:} For each $\mathcal{S}_i$, compute the ground truth bound as a set $\mathcal{C}_i \subseteq \mathcal{Y}$, such that $\forall s \in \mathcal{S}_i, \lambda(s) \in \mathcal{C}_i$.

The bounds computed this way satisfy the following condition:

\textit{Condition 4.2(b).} For any $x \in \tilde{\mathcal{X}}$, $\forall y \in \hat{f}(x), \exists \mathcal{X}_i $ such that $ x\in \mathcal{X}_i$ and $y \in \mathcal{C}_i$.

The idea behind \textit{Condition 4.2(b)} is the same as that of \textit{Condition 4.2(a)}, but formulated for discrete $y$. For tiles with $\mathcal{C}_i$ containing more than 1 class, we cannot verify correctness since there is more than 1 possible ground truth class for that tile. Therefore, we should try to make the state tiles containing only 1 ground truth class each when tiling the space in step 1. For tiles with $|\mathcal{C}_i| = 1$, we proceed to the following steps.
\\

\noindent\textbf{Step 3:} (compute bounding box for each input tile) is the same as regression.
\\

\noindent The next step is to solve the network output range. Suppose the quantity of interest has $K$ possible classes. Then the output layer of the neural network is typically a softmax layer with $K$ output nodes. Denote the $k$-th output score before softmax as $o_k(x)$. We use the solver to solve the minimum difference between the output score for the ground truth class and each of the other classes:

\noindent \textbf{Step 4:} Given $f$, $\mathcal{B}_i$, and the ground truth class $c_i$ (the only element in $\mathcal{C}_i$), use appropriate solver to solve lower bounds on the difference between the output score for class $c_i$ and each of the other classes: $l'^{(k)}_i$ for $k \in \{1,\dots , K\}\setminus \{c_i\}$. The bounds need to satisfy:

\textit{Condition 4.4(b)} $\forall x \in \mathcal{B}_i, \forall k \in \{1,\dots , K\}\setminus \{c_i\}, l'^{(k)}_i \leq o_{c_i}(x) - o_k(x)$.
\\

\noindent \textbf{Step 5:} For each tile, compute an error bound $e_i$, with $e_i=0$ meaning the network is guaranteed to be correct for this state tile, and $e_i=1$ meaning no guarantee:
\begin{equation}
\label{eqn:classification-tile-error}
    \text{$e_i = 0$ if and only if } \forall k \in \{1,\dots , K\}\setminus \{c_i\}, l'^{(k)}_i \geq 0
\end{equation}
Otherwise, $e_i = 1$. We can then compute the global and local error bounds using Equation \ref{eqn:regression-global} and \ref{eqn:regression-local}, same as in the regression case.

If we use SAT/SMT based solvers, we can combine step 4 and 5, formulating clauses on the network outputs: $\bigvee\limits_{k \in \{1,\dots , K\} \setminus \{c_i\}} o_{c_i}(x) - o_k(x) < 0$. If the resulting formula is UNSAT, then the correctness for this tile is verified.

\begin{theorem} [Local error bound for classification]
\label{theorem-class-local}
Given that Conditions 4.1, 4.2(b), 4.3, and 4.4(b) are satisfied, then $\forall x \in \tilde{\mathcal{X}}$, $e(x) \leq e_{\text{local}}(x)$, where $e(x)$ is defined in Equation \ref{eqn:error} and $e_{\text{local}}(x)$ is computed from Equation \ref{eqn:classification-tile-error} and \ref{eqn:regression-local}. Equivalently, when $e_{\text{local}}(x)=0$, the network prediction is guaranteed to be correct at $x$.
\end{theorem}

\begin{proof}
We aim to prove that $\forall x \in \tilde{\mathcal{X}}$, if $e_{\text{local}}(x)=0$, then $f(x) = \hat{f}(x)$. First, according to the definition of $e_{\text{local}}(x)$, $e_{\text{local}}(x)=0$ implies $e_i=0$ for all $i \in \{i | x\in \mathcal{B}_i\}$. Arbitrarily pick a $p \in \{i | x\in \mathcal{B}_i\}$, we have $e_{p}=0$. Then $\mathcal{C}_p$ only contains one element with class index $c_p$. Condition 4.4(b) gives:
\begin{align*}
\forall k \in \{1,\dots , K\}\setminus \{c_p\}, o_{c_p}(x) - o_k(x) \geq l'^{(k)}_p
\end{align*}
By Equation \ref{eqn:classification-tile-error}, $\forall k \in \{1,\dots , K\}\setminus \{c_p\}, l'^{(k)}_p \geq 0$. Therefore we have:
\begin{align*}
o_{c_p}(x) \geq o_k(x)
\end{align*}
for all $k \neq c_p$. This gives $f(x) = c_p$. For any $q \in \{i | x\in \mathcal{B}_i, i\neq p\}$, we can similarly derive $f(x) = c_q$. Therefore, we must have $c_q = c_p$ for all $q \in \{i | x\in \mathcal{B}_i, i\neq p\}$, otherwise there will be a contradiction. In another word, we have
\begin{align*}
    \bigcup_{\{i | x\in \mathcal{B}_i\}} \mathcal{C}_i &= \{c_p\}
\end{align*}

Now, according to Condition 4.2(b), $\forall y \in \hat{f}(x), \exists \mathcal{X}_i $ such that $ x\in \mathcal{X}_i$ and $y \in \mathcal{C}_i$. Rewriting this condition, we get
\begin{align*}
    \forall y \in \hat{f}(x), y \in \bigcup_{\{i | x\in \mathcal{X}_i\}} \mathcal{C}_i
\end{align*}
which also means $\hat{f}(x) \subseteq \bigcup_{\{i | x\in \mathcal{X}_i\}} \mathcal{C}_i$. But $\{i | x\in \mathcal{X}_i\} \subseteq \{i | x\in \mathcal{B}_i\}$ (Condition 4.3), which gives
\begin{align*}
    \hat{f}(x) \subseteq \bigcup_{\{i | x\in \mathcal{X}_i\}} \mathcal{C}_i \subseteq \bigcup_{\{i | x\in \mathcal{B}_i\}} \mathcal{C}_i &= \{c_p\}.
\end{align*}
Since $\hat{f}(x)$ is not empty, we have $\hat{f}(x) = \{c_p\} = f(x)$.
\end{proof}

\begin{theorem} [Global error bound for classification]
\label{theorem-class-global}
Given that Conditions 4.1, 4.2(b), 4.3, and 4.4(b) are satisfied, then if $e_{\text{global}}=0$, the network prediction is guaranteed to be correct for all $x \in \tilde{\mathcal{X}}$. $e_{\text{global}}$ is computed from Equation \ref{eqn:classification-tile-error} and \ref{eqn:regression-global}.
\end{theorem}

\begin{proof}
We aim to prove that if $e_{\text{global}}=0$, then $\forall x \in \tilde{\mathcal{X}}$, $f(x) = \hat{f}(x)$. According to Equation \ref{eqn:regression-global}, $e_{\text{global}}=0$ means $e_i=0$ for all $i$. Then $\forall x \in \tilde{\mathcal{X}}$, $e_{\text{local}}(x)=0$, which by Theorem \ref{theorem-class-local} indicates that $f(x) = \hat{f}(x)$.
\end{proof}

%\section{Algorithm}
%\label{sec:formal-algorithm-classification}

Algorithm \ref{alg:tiler-classification} formally presents the \textit{Tiler} algorithm for classification.

\begin{algorithm}
\small
\caption{Tiler (for classification)}
\label{alg:tiler-classification}
\begin{algorithmic}[1]
\Require $\mathcal{S}, g, \lambda, f$
\Ensure $e_{\text{global}}, \{e_i \}, \{\mathcal{B}_i \}$
     \Procedure{Tiler}{$\mathcal{S}, g, \lambda, f$}
          \State $\{\mathcal{S}_i \} \gets$ \Call{DivideStateSpace}{$\mathcal{S}$} \Comment{Step 1}
          \For{each $\mathcal{S}_i$}
            \State $c_i \gets$ \Call{GetGroundTruthBound}{$\mathcal{S}_i, \lambda$} \Comment{Step 2}
            \State $\mathcal{B}_i \gets$ \Call{GetBoundingBox}{$\mathcal{S}_i, g$} \Comment{Step 3}
            \State $\{l'^{(k)}_i\}_{k\neq c_i} \gets$ \Call{Solver}{$f, \mathcal{B}_i, c_i$} \Comment{Step 4}
            \State $e_i \gets$ \Call{GetErrorBound}{$\{l'^{(k)}_i\}_{k\neq c_i}$} \Comment{Step 5}
          \EndFor
          \State $e_{\text{global}} \gets$ {$\max (\{e_i \})$} \Comment{Step 5}
          \State \textbf{return} $e_{\text{global}}, \{e_i \}, \{\mathcal{B}_i \}$ \Comment{$\{e_i \}, \{\mathcal{B}_i \}$ can be used later to compute $e_{\text{local}}(x)$}
     \EndProcedure
\end{algorithmic}
\end{algorithm}

\section{Verification Subroutine}

After computing the ground truth bounds and bounding boxes, \textit{Tiler} uses some appropriate solver to verify that the maximum error in a tile is less than some threshold. \textit{Tiler} achieves this by either using some constrained optimization based method to solve network output ranges \cite{ILP2016,NSVerify2017,Cheng2017,tjeng2019,singh2018robustness,duality2018,ConvDual2018,certify2018}, or formulating clauses on network outputs and using SAT/SMT based methods to verify \cite{Reluplex2017,Planet2017,Huang2017,eevbnn}. We experiment with two verification subroutines in our case studies. The first is the MILP based method from \cite{tjeng2019}, which solves the network output ranges for real-valued networks exactly. The second is the SAT based
method from \cite{eevbnn}, which performs exact verification for binarized neural networks with a 10x--10000x speedup compared with MILP based exact verification methods on real-valued networks. We briefly describe these two techniques here.

\subsection{MILP-based Method}
\label{sec:milp}

The general principle is to formulate both input constraints and network computations as linear or integer constraints, set network outputs (or functions of network outputs) as the objective, and solve the resulting constrained optimization problem using MILP solver. Representing input ranges as inequality constraints is straightforward. For encoding the network computation, \cite{tjeng2019} introduce the following encoding for ReLU function $y = \max (x, 0), l\leq x \leq u$:
\begin{small}
\begin{equation}
  (y\leq x - l(1-a)) \land (y \geq x) \land (y \leq u\cdot a) \land (y \geq 0) \land (a \in \{0,1\})
\end{equation}
\end{small}
and encoding for the maximum function $y = \max(x_1, x_2, \cdots, x_m), l_i \leq x_i \leq u_i$:

\begin{footnotesize}
\begin{equation}
   \bigwedge\limits_{i=1}^m ((y \leq x_i + (1 - a_i)(u_{max,-i} - l_i)) \land (y \geq x_i)) \land (\sum_{i=1}^m a_i=1) \land (a_i \in \{0,1\})
\end{equation}
\end{footnotesize}
where $u_{max,-i} = \max_{j\neq i}(u_j)$.

Presolving on ReLU stability is used to reduce the number of constraints and thus speed up verification. This requires computing bounds on the inputs to non-linearities. \cite{tjeng2019} propose a progressive bound tightening technique to compute such bounds. It first uses faster methods to compute coarse bounds, and only refines bounds using slower but tighter methods if needed.

\subsection{SAT-based Verification of Binarized Neural Networks}
\label{sec:eevbnn}

\cite{eevbnn} considers binarized networks with ternary weights $\{-1, 0, 1\}$ and binary activations $\{0, 1\}$. The basic building block of a BNN is
\begin{equation}
  y = \textrm{bin}_{act}(\textrm{BatchNorm}(\textrm{linear}(x, \textrm{bin}_{w}(W)) ) )
\end{equation}
where
\begin{align*}
  \textrm{bin}_{act}(x) &= (\textrm{sign(x)} + 1) / 2 \\
  \textrm{bin}_w(x) &= \textrm{sign}(x) \\
  \textrm{linear} &\in \{ \textrm{convolution, matmul} \}
\end{align*}

The input is first quantized and then supplied to the BNN. The output layer does not have $\textrm{bin}_{act}$. It therefore outputs a real valued score for each class.

\paragraph{Encoding BNN verification as SAT formulae}

Each layer of the BNN is encoded as constraint:

\begin{equation}
  y = (k^{\textrm{BN}}\textrm{linear}(x, W^{\textrm{bin}}) + b^{\textrm{BN}} \geq 0 )
\end{equation}
where $W^{\textrm{bin}} = \textrm{bin}_w(W)$. Without loss of generality, consider $\textrm{linear}(x, W^{\textrm{bin}})$ as a dot product $\sum_{i=1}^n x_i W_i^{\textrm{bin}}$, then the constraint can be written as a reified cardinality constraint:

\begin{equation}
  y = \Big( \sum_{i=1}^n l_i \gtreqless b \Big)
\end{equation}
where

\begin{footnotesize}
\begin{align*}
  &b = - \frac{b^{\textrm{BN}}}{k^{\textrm{BN}}} - b^{\textrm{SAT}} \\
  &k^{\textrm{BN}}, b^{\textrm{BN}} \textrm{are BatchNorm parameters at inference time: } x^{\textrm{BN}} = k^{\textrm{BN}}x + b^{\textrm{BN}} \\
  &b^{\textrm{SAT}} = \sum_{i=1}^n \min (W_i^{\textrm{bin}}, 0) \\
  &l_i = \left\{
    \begin{array}{ll}
          x_i & \textrm{if } W_i^{\textrm{bin}} = 1 \\
          \neg x_i & \textrm{if } W_i^{\textrm{bin}} = -1 \\
    \end{array}
    \right. \\
  &\gtreqless \textrm{acts as } \geq \textrm{or } \leq \textrm{according to the sign of } k^{\textrm{BN}} \\
\end{align*}
\end{footnotesize}

For encoding the input constraints, consider a single dimension $x$ of the input that can vary within $[x_{\textrm{min}}, x_{\textrm{max}}]$ (generalizing to multiple dimensions is straightforward). The input after quantization can be encoded as:

\begin{equation}
  \nint{\frac{x}{s}} = L(x_0) + \sum_{i=1}^k t_i
\end{equation}
where
\begin{footnotesize}
\begin{align*}
  &s \textrm{ is the quantization step size} \\
  &k = U(x_0) - L(x_0) \\
  &L(x_0) = \nint{\frac{x_{\textrm{min}}}{s}} \\
  &U(x_0) = \nint{\frac{x_{\textrm{max}}}{s}} \\
  &\{ t_1, \cdots, t_k \} \textrm{ are Boolean variables that encodes the variation in } x \\
\end{align*}
\end{footnotesize}

As we have described in Section \ref{sec:classification}, the output verification condition can be encoded as:
\begin{equation}
\label{eqn:bnn-out}
  \vee_{k \in \{1,\dots , K\}\setminus \{c\}} r_{kc}, \textrm{where } r_{kc} = (o_{c}(x) - o_k(x) < 0)
\end{equation}
$c$ is the ground truth class. Each $r_{kc} = (o_{c}(x) - o_k(x) < 0)$ is also a reified cardinality constraint.

\paragraph{Performance enhancement}

\cite{eevbnn} introduces several techniques to speed up BNN verification. It implements a specialized SAT solver with native support for reified cardinality constraints. It maintains counters for the current number of known true or false literals for each clause, which facilitates the detection of cases that allow propagation for reified cardinality constraints.

The second enhancement is sparsifying the weights in the BNN by training a separate mask $M_W$ for the weights. The weight binarizing function defined earlier becomes:
\begin{equation}
  \textrm{bin}_w(W) = \textrm{sign}(W) \cdot \frac{\textrm{sign}(M_W) + 1}{2}
\end{equation}
This way of sparsifying results in a more balanced sparsity across layers, which is found to significantly improve the verification speed.

The third enhancement is adding a Cardinality Bound Decay (CBD) loss during training to encourage smaller bounds in reified cardinality constraints:
\begin{equation}
  L^{\textrm{CBD}} = \eta \max \Big( - \frac{b^{\textrm{BN}}}{k^{\textrm{BN}}} - b^{\textrm{SAT}} - \tau, 0 \Big)
\end{equation}

The experimental results show that BNN verification is 10-10000x faster than the verification of real-valued networks using MILP on MNIST and CIFAR10 \cite{eevbnn}.

\section{Position Measurement from Road Scene}

\subsection{Scene}
\label{sec:scene-append}

For clarity, we describe the scene in a Cartesian coordinate system. Treating 1 length unit as roughly 5 centimeters gives a realistic scale. The scene contains a road in the $\mathrm{xy}$-plane, extending along the $\mathrm{y}$-axis. A schematic view of the road down along the $\mathrm{z}$-axis is shown in Figure \ref{fig:road-scene}. The road contains a centerline and two side lines, each with width $\mathrm{x}_2=4.0$. The width of the road (per lane) is $\mathrm{x}_1=50.0$, measuring from the center of the centerline to the center of each side line. Each point in the scene is associated with an intensity value in $[0.0,1.0]$ grayscale. The intensity of the side lines is $i_1=1.0$, centerline $i_2=0.7$, road $i_3=0.3$, and sky $i_4=0.0$. The intensity adopts a ramp change at each boundary between the lines and the road. The half width of the ramp is $\mathrm{x}_3=1.0$.

The schematic of the camera is shown in Figure \ref{fig:camera}. The camera's height above the road is fixed at $\mathrm{z}_c=20.0$. The focal length $f=1.0$. The image plane is divided into $32\times32$ pixels, with pixel side length $d=0.16$.

\subsection{Camera Imaging Process}
\label{sec:camera-imaging-append}

The camera imaging process we use can be viewed as an one-step ray tracing: the intensity value of each pixel is determined by shooting a ray from the center of that pixel through the focal point, then taking the intensity of the intersection point between the ray and the scene. In the example scene, the intersection points for the top half of the pixels are in the sky (intensity 0.0). The intersection points for the lower half of the pixels are on the $\mathrm{xy}$-plane. The position of the intersection point (in world coordinates) can be computed using a transformation from the pixel coordinate in homogeneous coordinate systems: \footnote{Note that for this imaging process, flipping the image on the image plane to the correct orientation (originally upside-down) is equivalent to taking the image on a virtual image plane that is in front of the focal point by the focal length, by shooting rays from the focal point through the (virtual) pixel centers. We compute the intersection points from the pixel coordinates on the virtual image plane.}

\begin{equation}
    \Tilde{X}_W^{(i,j)} = T_{FW} \cdot
    P_p \cdot
    R_{CF} \cdot
    P_{PC} \cdot
    \begin{bmatrix}
    i \\
    j \\
    1
    \end{bmatrix}
\end{equation}

$P_{PC}$ is the transformation from pixel coordinates to camera coordinates. Camera coordinates have the origin at the focal point and axes aligned with the orientation of the camera. We define focal coordinates to have the origin also at the focal point, but with axes aligned with world coordinates (the coordinate system used in Section \ref{sec:scene-1}). $R_{CF}$ is the rotation matrix that transforms from camera coordinates to focal coordinates. $P_p$ represents the projection to the road plane through the focal point, in focal coordinates. Finally, $T_{FW}$ is the translation matrix that transforms from focal coordinates to world coordinates. The transformation matrices are given below:

\begin{small}
\begin{equation*}
    T_{FW} = \begin{bmatrix} 1&0&0&\delta \\
    0&1&0&0 \\
    0&0&1&\mathrm{z}_c\\
    0&0&0&1
    \end{bmatrix},
    P_p = \begin{bmatrix}
    -\mathrm{z}_c&0&0&0 \\
    0&-\mathrm{z}_c&0&0 \\
    0&0&-\mathrm{z}_c&0 \\
    0&0&1&0
    \end{bmatrix}
\end{equation*}

\begin{equation*}
    R_{CF} = \begin{bmatrix}
    \cos{\theta}&-\sin{\theta}&0&0 \\
    \sin{\theta}&\cos{\theta}&0&0 \\
    0&0&1&0 \\
    0&0&0&1
    \end{bmatrix},
    P_{PC} = \begin{bmatrix}
    0&d&\frac{d}{2} - \frac{nd}{2} \\
    0&0&f \\
    -d&0&-(\frac{d}{2} - \frac{nd}{2}) \\
    0&0&1
    \end{bmatrix}
\end{equation*}
\end{small}

The variables are defined as in Table \ref{tab:scene-params}, with $\delta$ and $\theta$ being the offset and angle of the camera. After the intensity values of the pixels are determined, they are scaled and quantized to the range $[0,255]$, which are used as the final image taken.

\begin{figure*}
    \centering
    \begin{subfigure}[b]{0.6\textwidth}
         \centering
         \includegraphics[width=0.8\textwidth]{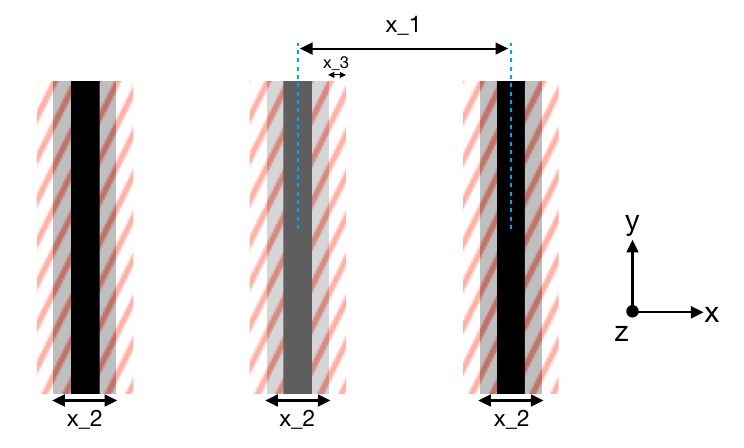}
         \caption{}
         \label{fig:road-scene}
     \end{subfigure}
     \begin{subfigure}[b]{0.35\textwidth}
         \centering
         \includegraphics[width=\textwidth]{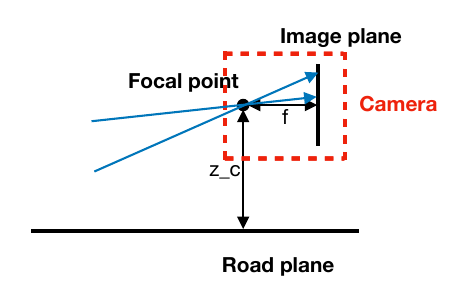}
         \caption{}
         \label{fig:camera}
     \end{subfigure}
    \caption{(a) Schematic of the top view of the road. (b) Schematic of the camera.}
    \label{fig:scene}
\end{figure*}

\begin{table*}
    \centering
    \caption{Scene parameters}
    \begin{tabular}{cccccc}
    \toprule
      Parameter & Description & Value & Parameter & Description & Value\\
    \midrule
    $\mathrm{x}_1$ & Road width & 50.0 & $i_1$ & Intensity (side line) & 1.0 \\
    $\mathrm{x}_2$ & Line width & 4.0 & $i_2$ & Intensity (centerline) & 0.7 \\
    $\mathrm{x}_3$ & Ramp half width & 1.0 & $i_3$ & Intensity (road) & 0.3\\
    $\mathrm{z}_c$ & Camera height & 20.0 & $i_4$ & Intensity (sky) & 0.0 \\
    $f$ & Focal length & 1.0 & $n$ & Pixel number & 32 \\
    $d$ & Pixel side length & 0.16 & & & \\
    \bottomrule
    \end{tabular}

    \label{tab:scene-params}
\end{table*}

\subsection{Method to Compute Pixel-wise Range for Each Tile}
\label{sec:compute-range}

In the road scene example, we need to encapsulate each tile in the input space with a $l_\infty$-norm ball for the MILP solver to solve. This requires a method to compute a range for each pixel that covers all values this pixel can take for images in this tile. This section presents the method used in this paper. 

A tile in this example corresponds to images taken with camera position in a local range $\delta \in [\delta_1, \delta_2]$, $\theta \in [\theta_1,\theta_2]$. For pixels in the upper half of the image, their values will always be the intensity of the sky. For each pixel in the lower half of the image, if we trace the intersection point between the projection ray from this pixel and the road plane, it will sweep over a closed region as the camera position varies in the $\delta$-$\theta$ cell. The range of possible values for that pixel is then determined by the range of intensities in that region. In this example, there is an efficient way of computing the range of intensities in the region of sweep. Since the intensities on the road plane only varies with $\mathrm{x}$, it suffices to find the span on $\mathrm{x}$ for the region. The extrema on $\mathrm{x}$ can only be achieved at: 1) the four corners of the $\delta$-$\theta$ cell; 2) the points on the two edges $\delta=\delta_1$ and $\delta=\delta_2$ where $\theta$ gives the ray of that pixel perpendicular to the $y$ axis (if that $\theta$ is contained in $[\theta_1,\theta_2]$). Therefore, by computing the location of these critical points, we can obtain the range of $\mathrm{x}$. We can then obtain the range of intensities covered in the region of sweep, which will give the range of pixel values.

\section{Sign Classification using LiDAR Sensor}

\subsection{Scene}
\label{sec:lidar-scene}

The sign is a planer object standing on the ground. It consists of a holding stick and a sign head on top of the stick. The stick is of height 40.0 and width 2.0 (all in terms of the length unit of the scene). The sign head has three possible shapes: square (with side length 10.0), equilateral triangle (with side length 10.0), and circle (with diameter 10.0). The center of the sign head coincides with the middle point of the top of the stick. 

The LiDAR sensor can vary its position within the working zone (see Figure \ref{fig:lidar-schematic}). Its height is fixed at 40.0 above the ground. The center direction of the LiDAR is always pointing parallel to the ground plane towards the centerline of the sign. 

\subsection{LiDAR Measurement Model}
\label{sec:lidar-model}

The LiDAR sensor emits an array of 32$\times$32 laser beams and measures the distance to the first object along each beam. The directions of the beams are arranged as follows. At distance $f=4.0$ away from the center of the sensor, there is a (imaginary) 32$\times$32 grid with cell size 0.1. There is a beam shooting from the center of the sensor through the center of each cell in the grid. 

The LiDAR model we consider has a maximum measurement range of \texttt{MAX\_RANGE}=300.0. If the distance of the object is larger than \texttt{MAX\_RANGE}, the reflected signal is too weak to be detected. Therefore, all the distances larger than \texttt{MAX\_RANGE} will be measured as \texttt{MAX\_RANGE}. The distance measurement contains a Guassian noise $n\sim \mathcal{N}(0, \sigma^2)$, where $\sigma = 0.001\times \texttt{MAX\_RANGE}$. So the measured distance for a beam is given by
\begin{equation*}
    d = d_0 + n, n\sim \mathcal{N}(0, \sigma^2)
\end{equation*}
\noindent where $d_0$ is the actual distance and $d$ is the measured distance. 

\subsection{Method to compute distance bounds for each beam}
\label{sec:lidar-box}

To compute the bounding boxes in \textit{Tiler}, we need to compute a lower bound and an upper bound on the measured distance for each beam as the sensor position varies within a tile. We first determine whether the intersection point $p$ between the beam and the scene is 1) always on the sign, 2) always on the background (ground/sky), or 3) covers both cases, when the sensor position varies within the tile. Denote the $d$ range of the tile as $[d_1, d_2]$ and the $\theta$ range as $[\theta_1, \theta_2]$. We find the intersection point $p$ at a list of critical positions: $(d_1, \theta_1)$, $(d_1, \theta_2)$, $(d_2, \theta_1)$, $(d_2, \theta_2)$. If at some angle $\theta^*\in [\theta_1, \theta_2]$, the beam direction has no component along the \texttt{y}-axis, then we add $(d_1, \theta^*)$ to the list of critical positions. The cases for $p$ (on sign/background) covered in the critical points determine the cases covered in the whole tile. 

In most situations, the distances of $p$ at the list of critical positions (we refer it as the list of critical distances) contain the maximum and minimum distances of $p$ in the tile. There is one exception: at $d=d_1$, as $\theta$ varies in $[\theta_1, \theta_2]$, if the intersection point shifts from sign to background (or vice versa), then the minimum distance of $p$ can occur at $(d_1, \theta')$ where $\theta'$ is not equal to $\theta_1$ or $\theta_2$. To handle this case, if at the previous step we find that $p$ do occur on the sign plane in the tile, then we add the distance of the intersection point between the beam and the \textit{sign plane} at position $(d_1, \theta_1)$ (or $(d_1, \theta_2)$, whichever gives a smaller distance) to the list of critical distances. Notice that this intersection point is not necessarily on the \textit{sign}. In this way, the min and max among the critical distances are guaranteed to bound the distance range for the beam as the sensor position varies within the tile. After this, the bounds can be extended according to the noise scale to get the bounding boxes. 

\section{Additional Training Details}
\label{sec:training-detail}

This section presents the additional details of the training of the neural network in the case study. We use Adam optimizer with learning rate $0.01$. We use early stopping based on the loss on the validation set: we terminate training if the validation performance does not improve in 5 consecutive epochs. We take the model from the epoch that has the lowest loss on the validation set.

%% file: samplepaper.bbl
\begin{thebibliography}{10}
\providecommand{\url}[1]{\texttt{#1}}
\providecommand{\urlprefix}{URL }
\providecommand{\doi}[1]{https://doi.org/#1}

\bibitem{fairsquare-17}
Albarghouthi, A., D'Antoni, L., Drews, S., Nori, A.V.: Fairsquare:
  Probabilistic verification of program fairness. Proc. ACM Program. Lang.
  \textbf{1}(OOPSLA),  80:1--80:30 (Oct 2017). \doi{10.1145/3133904},
  \url{http://doi.acm.org/10.1145/3133904}

\bibitem{CertGeometric2019}
Balunovic, M., Baader, M., Singh, G., Gehr, T., Vechev, M.: Certifying
  geometric robustness of neural networks. In: Wallach, H., Larochelle, H.,
  Beygelzimer, A., d\textquotesingle Alch\'{e}-Buc, F., Fox, E., Garnett, R.
  (eds.) Advances in Neural Information Processing Systems 32, pp.
  15313--15323. Curran Associates, Inc. (2019),
  \url{http://papers.nips.cc/paper/9666-certifying-geometric-robustness-of-neural-networks.pdf}

\bibitem{ILP2016}
Bastani, O., Ioannou, Y., Lampropoulos, L., Vytiniotis, D., Nori, A.V.,
  Criminisi, A.: Measuring neural net robustness with constraints. In:
  Proceedings of the 30th International Conference on Neural Information
  Processing Systems. pp. 2621--2629. NIPS'16, Curran Associates Inc., USA
  (2016), \url{http://dl.acm.org/citation.cfm?id=3157382.3157391}

\bibitem{fairness-19}
Bastani, O., Zhang, X., Solar-Lezama, A.: Probabilistic verification of
  fairness properties via concentration. Proc. ACM Program. Lang.
  \textbf{3}(OOPSLA),  118:1--118:27 (Oct 2019). \doi{10.1145/3360544},
  \url{http://doi.acm.org/10.1145/3360544}

\bibitem{bunel2017unified}
Bunel, R., Turkaslan, I., Torr, P.H.S., Kohli, P., Kumar, M.P.: A unified view
  of piecewise linear neural network verification (2017)

\bibitem{Cheng2017}
Cheng, C.H., N{\"u}hrenberg, G., Ruess, H.: Maximum resilience of artificial
  neural networks. In: ATVA (2017)

\bibitem{Zhu2017Adv}
Dong, Y., Liao, F., Pang, T., Hu, X., Zhu, J.: Discovering adversarial examples
  with momentum. CoRR  \textbf{abs/1710.06081} (2017),
  \url{http://arxiv.org/abs/1710.06081}

\bibitem{control2018}
Dutta, S., Jha, S., Sankaranarayanan, S., Tiwari, A.: Learning and verification
  of feedback control systems using feedforward neural networks.
  IFAC-PapersOnLine  \textbf{51}(16),  151 -- 156 (2018).
  \doi{https://doi.org/10.1016/j.ifacol.2018.08.026},
  \url{http://www.sciencedirect.com/science/article/pii/S240589631831139X}, 6th
  IFAC Conference on Analysis and Design of Hybrid Systems ADHS 2018

\bibitem{duality2018}
Dvijotham, K., Stanforth, R., Gowal, S., Mann, T., Kohli, P.: A dual approach
  to scalable verification of deep networks. In: Proceedings of the
  Thirty-Fourth Conference Annual Conference on Uncertainty in Artificial
  Intelligence (UAI-18). pp. 162--171. AUAI Press, Corvallis, Oregon (2018)

\bibitem{Planet2017}
Ehlers, R.: Formal verification of piece-wise linear feed-forward neural
  networks. CoRR  \textbf{abs/1705.01320} (2017),
  \url{http://arxiv.org/abs/1705.01320}

\bibitem{scenic-pldi19}
Fremont, D.J., Dreossi, T., Ghosh, S., Yue, X., Sangiovanni-Vincentelli, A.L.,
  Seshia, S.A.: Scenic: A language for scenario specification and scene
  generation. In: Proceedings of the 40th annual ACM SIGPLAN conference on
  Programming Language Design and Implementation (PLDI) (June 2019)

\bibitem{ai2-2018}
{Gehr}, T., {Mirman}, M., {Drachsler-Cohen}, D., {Tsankov}, P., {Chaudhuri},
  S., {Vechev}, M.: Ai2: Safety and robustness certification of neural networks
  with abstract interpretation. In: 2018 IEEE Symposium on Security and Privacy
  (SP). pp. 3--18 (May 2018). \doi{10.1109/SP.2018.00058}

\bibitem{Goodfellow2014}
Goodfellow, I.J., Shlens, J., Szegedy, C.: Explaining and harnessing
  adversarial examples. CoRR  \textbf{abs/1412.6572} (2014)

\bibitem{Huang2017}
Huang, X., Kwiatkowska, M., Wang, S., Wu, M.: Safety verification of deep
  neural networks. In: Majumdar, R., Kun{\v{c}}ak, V. (eds.) Computer Aided
  Verification. pp. 3--29. Springer International Publishing, Cham (2017)

\bibitem{verisig-18}
Ivanov, R., Weimer, J., Alur, R., Pappas, G.J., Lee, I.: Verisig: Verifying
  safety properties of hybrid systems with neural network controllers. In:
  Proceedings of the 22Nd ACM International Conference on Hybrid Systems:
  Computation and Control. pp. 169--178. HSCC '19, ACM, New York, NY, USA
  (2019). \doi{10.1145/3302504.3311806},
  \url{http://doi.acm.org/10.1145/3302504.3311806}

\bibitem{eevbnn}
Jia, K., Rinard, M.C.: Efficient exact verification of binarized neural
  networks. vol. abs/2005.03597 (2020)

\bibitem{Jin2020}
Jin, C., Rinard, M.C.: Manifold regularization for adversarial robustness.
  ArXiv  \textbf{abs/2003.04286} (2020)

\bibitem{Reluplex2017}
Katz, G., Barrett, C., Dill, D.L., Julian, K., Kochenderfer, M.J.: Reluplex: An
  efficient {SMT} solver for verifying deep neural networks. In: Computer Aided
  Verification, pp. 97--117. Springer International Publishing (2017).
  \doi{10.1007/978-3-319-63387-9\_5},
  \url{https://doi.org/10.1007%2F978-3-319-63387-9_5}

\bibitem{MCDC2001}
Kelly~J., H., Dan~S., V., John~J., C., Leanna~K., R.: A practical tutorial on
  modified condition/decision coverage. Tech. rep. (2001)

\bibitem{kingma2014}
Kingma, D.P., Ba, J.: Adam: {A} method for stochastic optimization. CoRR
  \textbf{abs/1412.6980} (2014), \url{http://arxiv.org/abs/1412.6980}

\bibitem{art-19}
Lin, X., Zhu, H., Samanta, R., Jagannathan, S.: {ART:} abstraction
  refinement-guided training for provably correct neural networks. CoRR
  \textbf{abs/1907.10662} (2019), \url{http://arxiv.org/abs/1907.10662}

\bibitem{verify-survey-2019}
Liu, C., Arnon, T., Lazarus, C., Barrett, C., Kochenderfer, M.J.: Algorithms
  for verifying deep neural networks. CoRR  \textbf{abs/1903.06758} (2019),
  \url{http://arxiv.org/abs/1903.06758}

\bibitem{NSVerify2017}
Lomuscio, A., Maganti, L.: An approach to reachability analysis for
  feed-forward relu neural networks. CoRR  \textbf{abs/1706.07351} (2017),
  \url{http://arxiv.org/abs/1706.07351}

\bibitem{DeepGauge2018}
Ma, L., Juefei{-}Xu, F., Sun, J., Chen, C., Su, T., Zhang, F., Xue, M., Li, B.,
  Li, L., Liu, Y., Zhao, J., Wang, Y.: Deepgauge: Comprehensive and
  multi-granularity testing criteria for gauging the robustness of deep
  learning systems. CoRR  \textbf{abs/1803.07519} (2018),
  \url{http://arxiv.org/abs/1803.07519}

\bibitem{Madry2017}
Madry, A., Makelov, A., Schmidt, L., Tsipras, D., Vladu, A.: Towards deep
  learning models resistant to adversarial attacks. ArXiv
  \textbf{abs/1706.06083} (2017)

\bibitem{TensorFuzz2018}
Odena, A., Goodfellow, I.J.: Tensorfuzz: Debugging neural networks with
  coverage-guided fuzzing. CoRR  \textbf{abs/1807.10875} (2018)

\bibitem{okelly2018scalable}
O'Kelly, M., Sinha, A., Namkoong, H., Duchi, J., Tedrake, R.: Scalable
  end-to-end autonomous vehicle testing via rare-event simulation (2018)

\bibitem{adversarial2016}
Papernot, N., McDaniel, P.D., Jha, S., Fredrikson, M., Celik, Z.B., Swami, A.:
  The limitations of deep learning in adversarial settings. 2016 IEEE European
  Symposium on Security and Privacy (EuroS\&P) pp. 372--387 (2016)

\bibitem{DeepXplore2017}
Pei, K., Cao, Y., Yang, J., Jana, S.: Deepxplore: Automated whitebox testing of
  deep learning systems. In: Proceedings of the 26th Symposium on Operating
  Systems Principles. pp. 1--18. SOSP '17, ACM, New York, NY, USA (2017).
  \doi{10.1145/3132747.3132785},
  \url{http://doi.acm.org/10.1145/3132747.3132785}

\bibitem{certify2018}
Raghunathan, A., Steinhardt, J., Liang, P.: Certified defenses against
  adversarial examples. In: International Conference on Learning
  Representations (2018), \url{https://openreview.net/forum?id=Bys4ob-Rb}

\bibitem{salman2019convex}
Salman, H., Yang, G., Zhang, H., Hsieh, C.J., Zhang, P.: A convex relaxation
  barrier to tight robustness verification of neural networks (2019)

\bibitem{NIPS2019Singh}
Singh, G., Ganvir, R., P\"{u}schel, M., Vechev, M.: Beyond the single neuron
  convex barrier for neural network certification. In: Wallach, H., Larochelle,
  H., Beygelzimer, A., d\textquotesingle Alch\'{e}-Buc, F., Fox, E., Garnett,
  R. (eds.) Advances in Neural Information Processing Systems 32, pp.
  15098--15109. Curran Associates, Inc. (2019),
  \url{http://papers.nips.cc/paper/9646-beyond-the-single-neuron-convex-barrier-for-neural-network-certification.pdf}

\bibitem{Singh2019}
Singh, G., Gehr, T., P\"{u}schel, M., Vechev, M.: An abstract domain for
  certifying neural networks. Proc. ACM Program. Lang.  \textbf{3}(POPL),
  41:1--41:30 (Jan 2019). \doi{10.1145/3290354},
  \url{http://doi.acm.org/10.1145/3290354}

\bibitem{singh2018robustness}
Singh, G., Gehr, T., Püschel, M., Vechev, M.: Robustness certification with
  refinement. In: International Conference on Learning Representations (2019),
  \url{https://openreview.net/forum?id=HJgeEh09KQ}

\bibitem{MCDCtest2018}
Sun, Y., Huang, X., Kroening, D.: Testing deep neural networks. CoRR
  \textbf{abs/1803.04792} (2018), \url{http://arxiv.org/abs/1803.04792}

\bibitem{adversarial2014}
Szegedy, C., Zaremba, W., Sutskever, I., Bruna, J., Erhan, D., Goodfellow,
  I.J., Fergus, R.: Intriguing properties of neural networks. CoRR
  \textbf{abs/1312.6199} (2014)

\bibitem{DeepTest2018}
Tian, Y., Pei, K., Jana, S., Ray, B.: Deeptest: Automated testing of
  deep-neural-network-driven autonomous cars. In: Proceedings of the 40th
  International Conference on Software Engineering. pp. 303--314. ICSE '18,
  ACM, New York, NY, USA (2018). \doi{10.1145/3180155.3180220},
  \url{http://doi.acm.org/10.1145/3180155.3180220}

\bibitem{tjeng2019}
Tjeng, V., Xiao, K.Y., Tedrake, R.: Evaluating robustness of neural networks
  with mixed integer programming. In: International Conference on Learning
  Representations (2019), \url{https://openreview.net/forum?id=HyGIdiRqtm}

\bibitem{Fastlin2018}
Weng, T.W., Zhang, H., Chen, H., Song, Z., Hsieh, C.J., Boning, D., Dhillon,
  I.S., Daniel, L.: Towards fast computation of certified robustness for relu
  networks. In: International Conference on Machine Learning (ICML) (july 2018)

\bibitem{ConvDual2018}
Wong, E., Kolter, Z.: Provable defenses against adversarial examples via the
  convex outer adversarial polytope. In: Dy, J., Krause, A. (eds.) Proceedings
  of the 35th International Conference on Machine Learning. Proceedings of
  Machine Learning Research, vol.~80, pp. 5286--5295. PMLR, Stockholmsmässan,
  Stockholm Sweden (10--15 Jul 2018),
  \url{http://proceedings.mlr.press/v80/wong18a.html}

\bibitem{wong2018}
Wong, E., Schmidt, F., Metzen, J.H., Kolter, J.Z.: Scaling provable adversarial
  defenses. In: Bengio, S., Wallach, H., Larochelle, H., Grauman, K.,
  Cesa-Bianchi, N., Garnett, R. (eds.) Advances in Neural Information
  Processing Systems 31, pp. 8400--8409. Curran Associates, Inc. (2018),
  \url{http://papers.nips.cc/paper/8060-scaling-provable-adversarial-defenses.pdf}

\bibitem{control-safety2018}
{Xiang}, W., {Tran}, H., {Rosenfeld}, J.A., {Johnson}, T.T.: Reachable set
  estimation and safety verification for piecewise linear systems with neural
  network controllers. In: 2018 Annual American Control Conference (ACC). pp.
  1574--1579 (June 2018). \doi{10.23919/ACC.2018.8431048}

\bibitem{exactreach-2017}
Xiang, W., Tran, H., Johnson, T.T.: Reachable set computation and safety
  verification for neural networks with relu activations. CoRR
  \textbf{abs/1712.08163} (2017), \url{http://arxiv.org/abs/1712.08163}

\bibitem{maxsens2018}
Xiang, W., Tran, H.D., Johnson, T.T.: Output reachable set estimation and
  verification for multi-layer neural networks. IEEE Transactions on Neural
  Networks and Learning Systems (TNNLS)  (Mar 2018).
  \doi{10.1109/TNNLS.2018.2808470}

\bibitem{shield-19}
Zhu, H., Xiong, Z., Magill, S., Jagannathan, S.: An inductive synthesis
  framework for verifiable reinforcement learning. In: Proceedings of the 40th
  ACM SIGPLAN Conference on Programming Language Design and Implementation. pp.
  686--701. PLDI 2019, ACM, New York, NY, USA (2019).
  \doi{10.1145/3314221.3314638},
  \url{http://doi.acm.org/10.1145/3314221.3314638}

\end{thebibliography}
